\documentclass[letterpaper]{article} % DO NOT CHANGE THIS
\usepackage{aaai23}  % DO NOT CHANGE THIS
\usepackage{times}  % DO NOT CHANGE THIS
\usepackage{helvet}  % DO NOT CHANGE THIS
\usepackage{courier}  % DO NOT CHANGE THIS
\usepackage[hyphens]{url}  % DO NOT CHANGE THIS
\usepackage{graphicx} % DO NOT CHANGE THIS
\usepackage{natbib}  % DO NOT CHANGE THIS AND DO NOT ADD ANY OPTIONS TO IT
\usepackage{caption} % DO NOT CHANGE THIS AND DO NOT ADD ANY OPTIONS TO IT
\usepackage{algorithm}
\usepackage[noend]{algorithmic}

\usepackage{newfloat}
\usepackage{listings}
\DeclareCaptionStyle{ruled}{labelfont=normalfont,labelsep=colon,strut=off} % DO NOT CHANGE THIS
\floatstyle{ruled}
\newfloat{listing}{tb}{lst}{}
\floatname{listing}{Listing}
\title{Direct Heterogeneous Causal Learning \\ for Resource Allocation Problems in Marketing}
\author{
    Hao Zhou\textsuperscript{\rm 1}\thanks{Corresponding author},
    Shaoming Li  \textsuperscript{\rm 1},
    Guibin Jiang  \textsuperscript{\rm 1},
    Jiaqi Zheng \textsuperscript{\rm 2},
    Dong Wang \textsuperscript{\rm 1}
}
\affiliations{
    \textsuperscript{\rm 1} Meituan, China \\
    \textsuperscript{\rm 2} State Key Laboratory for Novel Software Technology, Nanjing University, China \\
    \{zhouhao29, lishaoming, jiangguibin, wangdong07\}@meituan.com,  jzheng@nju.edu.cn

}

\usepackage{bibentry}
\usepackage{amsfonts}
\usepackage{amsmath}
\usepackage{amsthm}

\usepackage{booktabs} 
\usepackage{multirow}

\usepackage{graphicx}
\usepackage{float} 
\usepackage{subfigure}

\usepackage{enumerate}

\newtheorem{theorem}{Theorem}
\newtheorem{assumption}{Assumption}
\newtheorem{definition}{Definition}

\begin{document}

\maketitle

\begin{abstract}

Marketing is an important mechanism to increase user engagement and improve platform revenue, and heterogeneous causal learning can help develop more effective strategies. 
Most decision-making problems in marketing can be formulated as resource allocation problems and have been studied for decades. 
Existing works usually divide the solution procedure into two fully decoupled stages, i.e., machine learning (ML) and operation research (OR) --- the first stage predicts the model parameters and they are fed to the optimization in the second stage.
However, the error of the predicted parameters in ML cannot be respected and a series of complex mathematical operations in OR lead to the increased accumulative errors.
Essentially, the improved precision on the prediction parameters may not have a positive correlation on the final solution due to the side-effect from the decoupled design.

In this paper, we propose a novel approach for solving resource allocation problems to mitigate the side-effects. 
Our key intuition is that we introduce the decision factor to establish a bridge between ML and OR such that the solution can be directly obtained in OR by only performing the sorting or comparison operations on the decision factor.
Furthermore, we design a customized loss function that can conduct direct heterogeneous causal learning on the decision factor, an unbiased estimation of which can be guaranteed when the loss converges. 
As a case study, we apply our approach to two crucial problems in marketing: the binary treatment assignment problem and the budget allocation problem with multiple treatments. 
Both large-scale simulations and online A/B Tests demonstrate that our approach achieves significant improvement compared with state-of-the-art.

% Existed studies usually divide the solution procedure into two fully decoupled stages, i.e., predictive/uplift modelling and solving optimization problems. Under the circumstances, the error of the predictive model is usually ignored during solving optimization problems, and the model objective is far away from the original optimization objective. 
% The former may lead to the enlargement of the model error due to some complex mathematical operations in the combinatorial optimization algorithms, while the latter may cause that the improvement of the model prediction can not guarantee a better actual result of the original optimization problem.

% In this paper, we propose that the learning objective of a predictive model should depend on the specific combinatorial optimization algorithm. Based on this idea, we investigate three crucial problems in marketing: the cost-aware/unaware binary treatment assignment problem and the budget allocation problem with multiple treatments. For each problem, the key decision factor in the optimization algorithm is extracted and taken as the learning objective, which ensures that the objectives of the predictive model and the optimization problem are completely consistent. Furthermore, a global loss function with an unique extreme point is designed such that a direct prediction of the decision factor can be obtained when the loss converges. Both large-scale simulations and online A/B Tests shows that our approach achieves significant improvement compared with the state of the art.

\end{abstract}

\section{Introduction}

% \begin{figure}[htbp]
% \centering  %图片全局居中
% \includegraphics[width=0.49\textwidth]{figure/comparison.pdf}
% \caption{Comparison with two-phase methods}
% \label{fig:comparison}
% \end{figure}

Marketing is one of the most effective mechanisms for the improvement of user engagement and platform revenue. Thus, various marketing campaigns have been widely deployed in many online Internet platforms. 
For example, markdowns of perishable products in Freshippo are used to boost sales~\cite{Hua2021Markdowns}, coupons in Taobao Deals can stimulate user activity~\cite{Zhang2021BCORLE} and incentives in the Kuaishou video platform are offered to improve user retention~\cite{Ai2022LBCF}.

\begin{figure}[htbp]
\centering
\subfigure[Two-phase methods]{
\label{fig:TPM}
\includegraphics[width=\columnwidth]{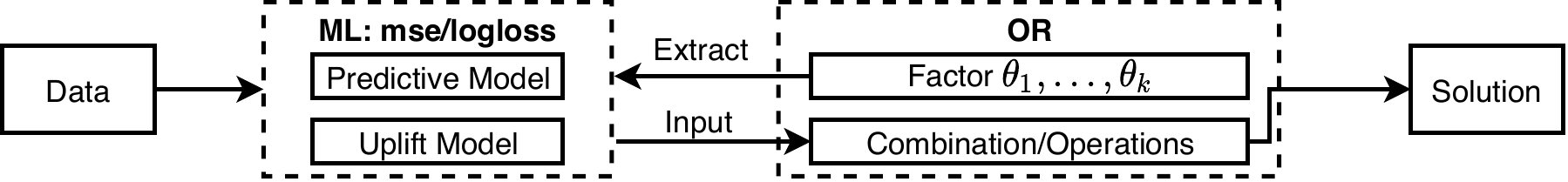}}
\centering
\subfigure[Our approach]{
\label{fig:decision-factor}
\includegraphics[width=\columnwidth]{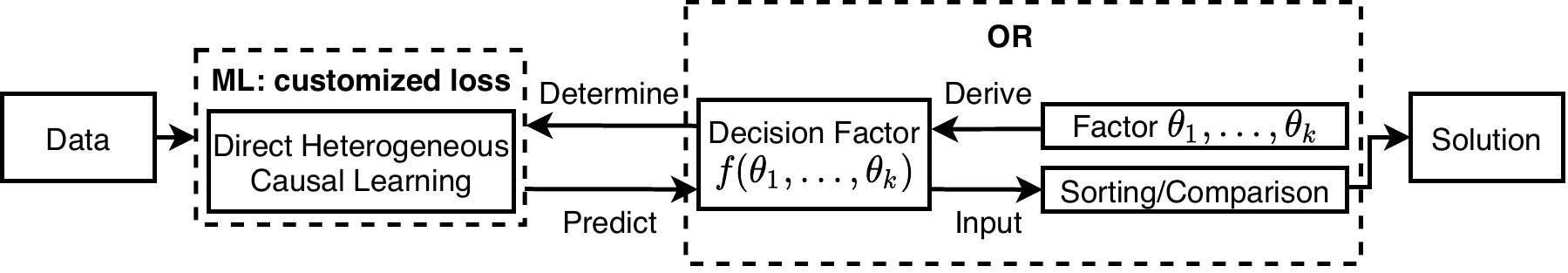}}
\caption{Comparison with two-phase methods}
\label{fig:comparison}
\end{figure}

Despite the incremental revenue, marketing actions can also incur significant consumption of marketing resources (e.g., budget). Hence, only part of individuals (e.g., shops or goods) may be assigned the marketing treatments due to the limited volume of them. 
In marketing, such decision-making problems can be formulated as resource allocation problems, and have been investigated for decades.

Most of existing studies usually use two-phase methods to solve these problems~\cite{Ai2022LBCF,Zhao2019Unified,Du2019Improve}. 
% In the first stage, the (incremental) response of individuals under different treatments is predicted by various machine learning models; In the second stage, the prediction results are taken as the input and the combinatorial optimization algorithms are developed to compute an approximately optimal solution of the original allocation problems. 
% Hence, existing works mainly focus on the separate optimization for predictive/uplift modeling and combinatorial optimization.
As is shown by Fig.~\ref{fig:TPM}, the first stage is ML, where the (incremental) response of individuals under different treatments is predicted by predictive/uplift models. The second stage is OR and the prediction results in ML are fed as the input to the combinatorial optimization algorithms.
Hence, existing works mainly focus on the decoupled optimization for predictive/uplift modeling and combinatorial optimization.

Despite the widespread application, there are two major defects in two-phase methods. 
The first one is that the solution is obtained after conducting many intermediate computations on the prediction results in ML, e.g., the combination of multiple factors or complex mathematical operations in OR. 
% The long distance between the solution and ML causes that the improvement of the response/uplift model cannot guarantee a better business result.
Therefore, the improved precision on the prediction parameters may not have a positive correlation on the final solution.
The second is that the errors of model prediction are not respected, and the complex operations performed on prediction results in OR lead to the increased accumulative errors.
Due to the accumulative errors, the theoretically optimal algorithm in OR cannot always achieve the practically optimum and is even inferior to heuristic strategies in some scenarios.
Therefore, the decoupled optimization for ML and OR cannot induce a global optimization for the original problem.

% Formally, the decision factor is defined as a combination object on which the solution in the optimization algorithm can be directly obtained by only the sorting or comparison operations.
% % In this paper, we propose that the learning objective of a predictive model should depend on the specific combinatorial optimization algorithm. 
% Instead of a complete decoupling of predictive/uplift modeling and combinatorial optimization, the decision factor in the optimization algorithm is extracted and taken as the learning objective in this paper. Making a direct prediction for it has two important advantages. Firstly, the solution is directly obtained based on it, which avoids a further enlargement of model errors caused by complex mathematical operations. Secondly, the prediction results directly determine the allocation of marketing actions, which ensures that the model objective and the optimization objective are completely consistent. Hence, the new challenges are how to identify such a decision factor for the specific optimization algorithm and how to make a direct prediction.

Instead of two-phase methods, we propose a novel approach for solving resource allocation problems to mitigate the above defects. 
First of all, we define the decision factor of an algorithm as a factor on which the solution can be directly obtained by performing only the sorting or comparison operations. As is shown by Fig.~\ref{fig:decision-factor}, the decision factor derived from OR is taken as the learning objective, and direct heterogeneous causal learning in ML is conducted in our approach. 
Based on the definition, there is no alternative mathematical operations on the prediction results in OR. 
Therefore, the ranking performance of one model on the decision factor directly determines the quality of the solution and improving the model can guarantee a better solution. 
Specifically, the model error can be used to measure the ranking performance instead, which is respected and not enlarged in OR.
Hence, the new challenges are how to identify such a decision factor in OR and how to make a direct prediction for it in ML.

Following this idea, we investigate two crucial problems in marketing. The first one is the binary treatment assignment problem. When ignoring the cost incurred by the treatment (the cost-unaware version), the conditional average treatment effect (CATE) can be regarded as the decision factor. The common uplift models to predict CATE include meta-learners~\cite{Kunzel2019Metalearners,Nie2021Quasi} and causal forests~\cite{Wager2018Estimation,Athey2019Generalized}.
The former are composed of multiple base models, and the latter usually combines generalized random forests (GRF) with double machine learning (DML) methods. 
Different from them, we propose a novel uplift model to make a direct prediction based on neural networks, which achieves good performance in both theory and practice. 
Despite the incremental revenue, the treatment can also incur different costs. In this cost-aware version, ROI (Return on Investment) of individuals can be regarded as the decision factor, which is calculated by the division of the incremental revenue and the incremental cost. However, most of existing works in causal inference did not involve the treatment cost and cannot apply to such a direct prediction. 
Although some works~\cite{Du2019Improve} investigated a similar problem to this, their loss function cannot converge to a stable extreme point in theory.
In this paper, we design a convex loss function to guarantee an unbiased estimation of ROI of individuals when the loss converges.

% As the second case study, the budget allocation problem with multiple treatments is investigated and a novel evaluation metric for this problem is proposed in this paper. 
As the second case study, we apply our approach to the budget allocation problem with multiple treatments and propose a novel evaluation metric for this problem in this paper. 
% The second problem is the budget allocation problem with multiple treatments.
% , which is one of the most common marketing scenarios and has been considered in many existing works. 
% Instead of two-phase method, we first focus on the combinatorial optimization algorithm. 
The Lagrange duality is an effective algorithm to solve the budget allocation problem.
% and has been widely used in many related studies. 
However, the decision factor in this algorithm contains the Lagrange multiplier that is uncertain and varies much with different budgets. The direct prediction for such a decision factor with all possible Lagrange multipliers is difficult and unrealistic. In this paper, we propose an equivalent algorithm to the Lagrange dual method while the decision factor in this algorithm is determined and irrelevant to the Lagrange multiplier. In addition, the corresponding causal learning model is developed and a direct prediction for the decision factor can be obtained when the customized loss function converges. Finally, we also propose a novel evaluation metric named MT-AUCC to estimate the prediction result, which is similar to Area Under Uplift Curve (AUUC)~\cite{Rzepakowski020Decision} but involves both multiple treatments and incremental cost.
%However, most of the solutions are based on the two-phase methods, which consists of a multiple-treatment uplift/response model and the Lagrange dual algorithm. 

Large-scale simulations and online A/B Tests validate the effectiveness of our approaches. In the offline simulations, we use two real-world datasets collected from random control trials (RCT) in the online advertising/food delivery platforms. 
% Multiple evaluation metrics show that our models and algorithms achieve significant improvement compared with existing works. They are also deployed in the real online platform to verify their performance. Online AB tests show that our approach increases the target reward by over 10\% compared to the state of the art. 
Multiple evaluation metrics and online AB tests show that our models and algorithms achieve significant improvement and increase the target reward by over 10\% on average compared with state-of-the-art.

\section{Related Work}

{\bf Two-phase methods.} The composition of machine learning (ML) and operation research (OR) is one of the most common approaches to solve the resource allocation problem, which is called two-phase methods in this paper. In the first stage, uplift models are designed to predict the incremental response of individuals under different treatments. 
% Existing uplift modeling methods mainly contain meta-learners~\cite{Kunzel2019Metalearners,Nie2021Quasi}, causal forests~\cite{Wager2018Estimation,Athey2019Generalized} and representation learning~\cite{Johansson2016Learning,Shalit2017Estimating,Yao2018Representation}. 
Besides meta-learners~\cite{Kunzel2019Metalearners,Nie2021Quasi} and causal forests~\cite{Wager2018Estimation,Athey2019Generalized,Zhao2017Uplift,Ai2022LBCF}, representation learning~\cite{Johansson2016Learning,Shalit2017Estimating,Yao2018Representation} was also developed for uplift modeling.
% Heterogeneous causal inference with multiple treatments was also investigated in recent years, most of which are an extension of causal tree models~\cite{Zhao2017Uplift,Ai2022LBCF}. 
Instead of deriving an unbiased estimator, some works~\cite{Betlei2021Uplift,Kuusisto2014Support} proposed a unified framework for learning to rank CATE.
As one of the most effective algorithms, the Lagrangian duality is frequently used to solve many decision-making problems of different areas in the second stage.
% In the second stage, the combinatorial optimization algorithm is developed to obtain an approximately optimal solution. The Lagrangian duality is an effective algorithm in OR, and has been proposed to solve many decision-making problems of different areas. 
% In marketing, the works adopted this method to deal with binary treatment scenarios~\cite{Du2019Improve}, and it was expanded into multiple treatments and continuous treatments in other studies~\cite{Ai2022LBCF,Zhao2019Unified}. In online advertising, it was developed to solve dynamic knapsack problems and compute the optimal bidding policy~\cite{Hao2020Dynamic}. 
For example, it was developed to solve the budget allocation problem in marketing~\cite{Du2019Improve,Ai2022LBCF,Zhao2019Unified} and compute the optimal bidding policy in online advertising~\cite{Hao2020Dynamic}.

\noindent{\bf Direct learning methods.} Policy learning and reinforcement learning are two important methods to directly learn a treatment assignment policy rather than the treatment effect, which avoids the combination of ML and OR. A general framework for policy learning with observational data was proposed based on the doubly robust estimator~\cite{Athey2021Policy}
% , which can establish a theoretical regret guarantee for the resulting policy. 
and their work was extended to multi-action policy learning~\cite{Zhou2022Offline}.
% and a substantial performance improvement is provided~\cite{Zhou2022Offline}. 
% The coupon allocation problem in sequential incentive marketing can be formulated as a constrained Markov decision process and reinforcement learning is developed to solve this problem~\cite{Xiao2019Model,Zhang2021BCORLE}.
As a real-world application, the works~\cite{Xiao2019Model,Zhang2021BCORLE} formulated the coupon allocation problem in sequential incentive marketing as a constrained Markov decision process and proposed reinforcement learning to solve it. 
However, all of the above methods moved the resource constraints into the reward function by using the Lagrangian multiplier. Hence, the model may need to be changed constantly with the variation of the Lagrangian multiplier.
% For the resource allocation problem, all of the above methods moved the resource constraints into the reward function by using the Lagrangian multiplier, which causes that the model may need to be changed with the variation of the Lagrangian multiplier.

\noindent{\bf Decision-focused learning (DFL).} Similar to our motivation, DFL devotes itself to learning the model parameters based on the downstream optimization task rather than the prediction accuracy. Nevertheless, many existing works in DFL required that the feasible region of the decision variables is fixed and known with certainty~\cite{Bryan2019DFL,Adam2022SPO,Shah2022LODL,Mandi2022Decision}. The most related work to ours is perhaps by~\citeauthor{Donti2017task}~\citeyear{Donti2017task}, which addressed a stochastic optimization problem that contains both probabilistic and deterministic constraints. However, this work supposed that the decision variables are continuous and dealt with the probabilistic constraints by the Lagrangian duality, which is markedly distinct from ours. 
% \noindent{\bf Our work} fall in between two-phase methods and direct learning methods. Compared with the former, the decision factor in the optimization algorithm is extracted as the learning objective, which avoids an enlargement of model errors and guarantees the consistency between the model objective and the optimization objective; Compared with the latter, we do not directly learn the optimal policy, which dramatically decreases the hardness of model prediction.

\section{Binary Treatment Assignment Problem}

% In this section, we will introduce a simple optimization problem in marketing. The decision factor directly related to the optimization objective can be obtained easily and we will show how to make a direct prediction for it. 

We begin with a common marketing scenario, where part of $M$ individuals are selected to receive the marketing action. 
We adopt the potential outcome framework~\cite{sekhon2008neyman} to formulate this problem. 
Let $X \in \mathbb{R}^d$ denote the feature vector and x its realization. 
Despite the incremental revenue, marketing actions can also incur significant costs. Let $Y^r,Y^c$ denote the revenue and the cost respectively, and $y^r, y^c$ its realization. 
Denote the treatment by $T \in \{0,1\}$ and its realization by $t_i$. Let $(Y^r(1), Y^r(0))$ and $(Y^c(1), Y^c(0))$ be the corresponding potential outcome when the individual receives the treatment or not. Define $\tau^r(x_i), \tau^c(x_i)$ as the conditional average treatment effect, which can be calculated by 
$$\tau^*(x_i) = E[Y^*(1)-Y^*(0)|X=x_i], * \in \{r,c\}.$$
Since most marketing actions have a positive effect on the response of an individual, we have $\tau^r(x_i) > 0$ and $\tau^c(x_i) > 0$.
The binary treatment assignment problem (BTAP) is to assign the treatment to part of the individuals to maximize the overall revenue on the platform, but requires that the incremental cost does not exceed a limited budget~$B$. Let $z_i \in \{0,1\}$ be the decision variables. Therefore, BTAP can be formulated as the integer programming problem~\eqref{BTAP}. 

\begin{align}
\max& \sum_i z_i \tau^r(x_i) \label{BTAP}\\
s.t.& \sum_i z_i \tau^c(x_i) \le B& \nonumber\\
& z_i \in \{0,1\}, \forall i.& \nonumber
\end{align}

This problem is an equivalent 0/1 knapsack problem, which is NP-Hard. Fortunately, the simple Greedy algorithm (Algorithm~\ref{alg:BTAG}) can achieve excellent performance whose approximation ratio satisfies $\rho \ge 1 - \frac{\max_i \tau(x_i)}{\mathrm{OPT}}$, where $\mathrm{OPT}$ is the optimal solution of the problem~\eqref{BTAP}.

\begin{definition}\label{def:decision-factor}
The decision factor of a combinatorial optimization algorithm is defined as a factor on which the final solution can be obtained by performing only the sorting or comparison operations.
\end{definition}

The decision factor is directly to the final solution of an algorithm and will be regarded as the learning objective in this paper. As is shown by Algorithm~\ref{alg:BTAG}, the factor $\tau^r(x_i)/\tau^c(x_i)$ can be taken as the decision factor, which is called as ROI (Return on Investment) of individual $i$. 
% Therefore, the better prediction for ROI or the rank of ROI, the higher revenue of the original optimization problem.

\begin{algorithm}[tb]
\caption{Greedy Algorithm for BTAP}
\label{alg:BTAG}
\textbf{Input}: $(\tau^r(x), \tau^c(x), B, M)$\\
\textbf{Output}: the solution to BTAP
\begin{algorithmic}[1] %[1] enables line numbers
\STATE Sort $M$ individuals in descending order based on $\frac{\tau^r(x_i)}{\tau^c(x_i)}$.
\STATE Assign the treatment to individual $i$ in the sorted list until the budget~$B$ will be exceeded.

\end{algorithmic}
\end{algorithm}

\subsection{Cost-unaware Treatment Assignment Problem}

When the treatment cost is nonexistent or the same for all~the individuals (e.g., a push message), the prediction for ROI of individuals is reduced to the estimation of $\tau^r(x_i)$. 
The latter is an important problem in causal inference, and existing studies mainly involve meta-learners and causal forests. In this paper, we propose a novel uplift model to make a direct prediction for $\tau^r(x_i)$ or the rank of $\tau^r(x_i)$, which can achieve good performance in both theory and practice.

Following the above notations, suppose that there is a data set of size $N$ collected from random control trials (RCT) and denote the i-th sample by $(x_i, t_i, y^r_i)$.  Denote the sample size that receive the treatment or not by $N_1$ and $N_0$, respectively. Let $s_i = \hbar(x_i)$ represent a score to rank $\tau^r(x_i)$, where $\hbar(x_i)$ can be any machine learning model (e.g., linear regression or neural networks). Minimizing the loss function \eqref{DUM}, we can obtain an unbiased estimation of $\tau^r(x_i)$. 
% The detailed analysis can be obtained in Theorem~\ref{theorem:DUM}. Due to the space limit, its proof can be found in Appendix A.
Due to the space limit, the detailed analysis is presented by Theorem~\ref{theorem:DUM} and its proof in Appendix~A~\cite{Zhou2022DHCL}.

\begin{align}
\min& L(s) = - (\frac{1}{N_1}\sum_{i|t_i=1} y^r_i \ln q_i -\frac{1}{N_0}\sum_{i|t_i=0} y^r_i \ln q_i) \label{DUM}\\
s.t.& \ \ \ \ \ \ \ \ \ \ \ \ \ \ \ \ q_i = \frac{e^{s_i}}{\sum_i e^{s_i}}, \forall i& \nonumber\\
& \ \ \ \ \ \ \ \ \ \ \ \ \ \ \ \ s_i = \hbar(x_i) \in \mathbb{R}, \forall i& \nonumber
\end{align}

\begin{theorem} \label{theorem:DUM}
When the loss function~\eqref{DUM} converges, $s_i$ can be used to rank $\tau^r(x_i)$ and $q_i = \frac{\tau^r(x_i)}{\sum_i \tau^r(x_i)}$ can be used to obtain an unbiased estimation of $\tau^r(x_i)$. 
\end{theorem}

% \begin{proof}
% In the loss \eqref{DUM}, $q_i$ is a function of $x_i$ and denote it by $q_i = \kappa(x_i)$. 
% Therefore, we have 
% \begin{align*}
% \frac{1}{N_1}\sum_{i|t_i=1} y_i \ln q_i &= E[Y_i \ln \kappa(X_i) | T_i = 1]\\
% &= E[Y_i(1) \ln \kappa(X_i) | T_i = 1] \\
% &= E[Y_i(1) \ln \kappa(X_i)].
% \end{align*}
% Hence, $L(s)$ can be rewritten as the following.
% \begin{align*}
% L(s) =& -(E[Y_i(1) \ln \kappa(X_i)] - E[Y_i(0) \ln \kappa(X_i)]) \\
% =& -E[(Y_i(1) - Y_i(0)) \ln \kappa(X_i)] \\
% =& - E_X[E[(Y(1)-Y(0)) \ln \kappa(X) | X]] \\
% =& - E_X[E[\tau(X) \ln \kappa(X) | X]] \\
% =& - E[\tau(X_i) \ln \kappa(X_i)] \\ 
% =& - \frac{1}{N} \sum_i \tau(x_i) \ln q_i.
% \end{align*}
% When the loss converges, we have
% \begin{align*}
% \frac{\partial L(s)}{\partial s_j} =& - \frac{1}{N} \sum_i \frac{\tau(x_i)}{q_i} \frac{\partial q_i}{\partial s_j} \\
% =& - \frac{1}{N} [\tau(x_j) - q_j \sum_i \tau(x_i)] = 0.
% %=& 0.
% \end{align*}
% It implies that $q_j = \frac{\tau(x_j)}{\sum_i \tau(x_i)}$ holds for $\forall j$. According to the loss function~\eqref{DUM}, $q_i < q_j$ can be derived from $s_i < s_j$, where the former means $\tau(x_i) < \tau(x_j)$. 
% \end{proof}

\subsection{Cost-aware Treatment Assignment Problem}

ROI of individuals is a composite object and most existing works predict it by the combination of multiple models. 
% Like T-Learner in uplift models, the higher precision of the base model is not equivalent to the improvement of the prediction for the composite object. 
The latter may cause an enlargement of model errors due to the mathematical operations during combination.
Therefore, we propose a novel learning model for direct ROI prediction.

\begin{align}
\min& L(s) = - [\frac{1}{N_1}\sum_{i|t_i=1} (y^r_i \ln \frac{q_i}{1-q_i} + y^c_i \ln (1 - q_i)) - \nonumber\\ &\ \ \ \ \ \ \ \ \ \ \ \  \frac{1}{N_0}\sum_{i|t_i=0} (y^r_i \ln \frac{q_i}{1-q_i} + y^c_i \ln (1 - q_i))] \label{DRP}\\
s.t.& \ \ \ \ \ \ \ \ \ \ \ \ \ \ \ \ q_i = \sigma(s_i), \forall i& \nonumber\\
& \ \ \ \ \ \ \ \ \ \ \ \ \ \ \ \ s_i = \hbar(x_i) \in \mathbb{R}, \forall i.& \nonumber
\end{align}

Similarly, suppose that the data set $(x_i, t_i, y^r_i, y^c_i)$ of size $N$ is collected from RCT, and the count of the samples receiving the treatment or not is $N_1$ and $N_0$, respectively. 
% The division operation may result in the high variance of ROI, and there may be some individuals with huge ROI especially when $\tau^c(x_i)$ of them are small. 
The division operation may result in the high variance of ROI especially when $\tau^c(x_i)$ is small.
Therefore, the range of ROI is limited to $(0,1)$ by the scaling and truncating operations for $Y^r$ or $Y^c$ to decrease the risk of overfitting. Let $s_i = \hbar(x_i)$ represent a score to rank ROI, where $\hbar(x_i)$ can be any machine learning model. Define $\sigma(\cdot)$ as the sigmoid function. 
% Based on the notations, 
The loss function~\eqref{DRP} is designed to obtain an unbiased estimation of ROI or the rank of ROI for each individual. The detailed proof can be found in Theorem~\ref{theorem:DRP} and Appendix~B~\cite{Zhou2022DHCL}.

\begin{theorem} \label{theorem:DRP}
The loss function~\eqref{DRP} is convex , and $s_i$ can be used to rank ROI and $q_i = \frac{\tau^r(x_i)}{\tau^c(x_i)}$ is an unbiased estimation of ROI of individual $i$ when the loss converges. 
\end{theorem}

\section{Budget Allocation Problem with Multiple Treatments}

In this section, we will discuss a more general marketing scenario, which was also introduced in many existing works. 
% Most of them solved this problem by means of predictive/uplift modeling and combinatorial optimization. However, the two stages are fully decoupled, which may further cause that the model error is ignored during combinatorial optimization and the prediction objective of the model is far away from the optimization objective of the original problem.
Most of them solve this problem by two-phase methods.
Different from them, the machine learning model will be designed based on the decision factor to guarantee the consistency between the predictive objective and the optimization objective in this paper.

Suppose that there are multiples treatments and denote it by $T \in \{1,2,\ldots,L\}$. Let $r_{ij}$ and $c_{ij}$ be the revenue and the cost of individual $i$ receiving treatment $j$, respectively. In marketing campaigns, multiple treatments usually refers to the different level of the treatment, e.g., the different discount of some products [1], the different count of gold pieces on the online video platform [2], the different price of cinema tickets in different regions [3] and so on. Suppose that $j \in T$ represent the level of the marketing treatment and the treatment effect is larger if the level is higher. Therefore, we have $r_{ij} < r_{ik}$ and $c_{ij} < c_{ik}$ if $j < k$ holds. Given a limited budget $B$, the budget allocation problem with multiple treatments (MTBAP) is to assign a certain treatment to each individual with the objective of optimizing the overall revenue on the platform. Let $z_{ij} \in \{0,1\}$ be the decision variable to denote whether to assign treatment $j$ to individual $i$. Therefore, MTBAP can be formulated as the integer programming~\eqref{MTBAP}.

\begin{align}
\max& \sum_{ij} z_{ij} r_{ij} \label{MTBAP}\\
s.t.& \sum_{ij} z_{ij} c_{ij} \le B&\nonumber \\
& \sum_j z_{ij} = 1, \forall i&\nonumber \\
& z_{ij} \in \{0,1\}, \forall i,j.&\nonumber 
\end{align}

The first constraint refers to the limitation of the whole budget and the second requires that only one treatment can be assigned to each individual.

\subsection{Combinatorial Optimization Algorithm}

This problem is a classical multiple choice knapsack problem and remains NP-Hard. Existing studies usually solve this problem by using Lagrangian duality theory. Specifically, the upper bound of the original problem~\eqref{MTBAP} can be obtained by solving the following dual problem~\eqref{MTBAP-Dual}.

\begin{align}
&\min_{\alpha \ge 0} \left (
\begin{array}{c}
\max \alpha B + \sum_{ij} z_{ij} (r_{ij} - \alpha c_{ij}) \\
s.t. \ \sum_j z_{ij} = 1, \forall j \\
z_{ij} \in \{0,1\}, \forall i, j
\end{array}
\right ) \nonumber\\
= & \min_{\alpha \ge 0} (\alpha B + \sum_i \max_j (r_{ij} - \alpha c_{ij})) \label{MTBAP-Dual}
\end{align}

The optimal $\alpha^*$ for the dual problem can be derived by using the decreasing gradient algorithm or a binary search for $\alpha$ with the terminal condition of $B - \sum_{ij} z_{ij} c_{ij} \le \epsilon$. Given the optimal $\alpha^*$, an approximation solution for the original problem~\eqref{MTBAP} is
\begin{equation} \label{eq:dual-solution}
\forall i,j, \ \ z_{ij} = 1 \iff j = \arg\max_j r_{ij} - \alpha^* c_{ij}.
\end{equation}
% In order to guarantee the consistency between model prediction and combinatorial optimization, the value or the rank of the decision factor (i.e., $r_{ij} - \alpha^* c_{ij}$) in the optimization algorithm should be regarded as the prediction objective. 
According to Definition~\ref{def:decision-factor}, the factor $r_{ij} - \alpha^* c_{ij}$ is a decision factor for the Lagrangian duality algorithm and can be taken as the learning objective.
However, the value of $\alpha^*$ depends on the budget $B$,  which is undetermined and varies much with the marketing environment. Hence, there is only a limited number of $\alpha^*$ in the training dataset and it is possible that the optimal $\alpha^*$ used in future campaigns has never been seen by the predictive model before, which dramatically decreases the precision of the model. %In addition, the rank of the decision factor under the same $\alpha^*$ should be of great concern to the model, instead of that between different $\alpha^*$. 
Therefore, the direct prediction of the decision factor for all possible $\alpha^*$ is difficult and unrealistic, and will not be considered in this paper.

%Instead of Lagrangian duality methods, we attempt other combinatorial optimization algorithms. 

However, another equivalent optimization algorithm can be derived from the above Lagrangian duality method. 
Before the details, we present an important assumption in economics, which is called as the Law of Diminishing Marginal Utility~\cite{polleit2011can}.
% Before the details, we present an important assumption called as the Law of Diminishing Marginal Utility in economics.

\begin{assumption}[The Law of Diminishing Marginal Utility]\label{assumption:law}
The marginal utility of individuals decreases with the increasing investment of the marketing cost. Specifically, denote the marginal utility by 
$$\ell_{ij} = \frac{r_{ij+1} - r_{ij}}{c_{ij+1} - c_{ij}},$$
and we have
$\ell_{ij} \le \ell_{ij-1}, \forall i,j.$
%$$\frac{r_{ij+1} - r_{ij}}{c_{ij+1} - c_{ij}} \le \frac{r_{ij} - r_{ij-1}}{c_{ij} - c_{ij-1}}, \ \forall i,j.$$
\end{assumption}

% \noindent Let $(c_{ij}, r_{ij})$ be a point in rectangular coordinates and we can get the following two important properties.

% \begin{enumerate}
% \item Due to $r_{ij} - \alpha * c_{ij} = (c_{ij}, r_{ij}) \cdot (-\alpha, 1)$, the value of $r_{ij} - \alpha * c_{ij}$ can be regarded as the projection of $(c_{ij}, r_{ij})$ on $(-\alpha, 1)$.
% \item Based on Assumption~\ref{assumption:law}, the points $(c_{ij}, r_{ij})$ for each $i$ constructs a convex hull.
% \end{enumerate}

%  Let $j^*$ be the solution of Lagrangian duality methods in Eq.~\eqref{eq:dual-solution}, i.e., $j^* = \arg \max_j r_{ij} - \alpha c_{ij}$. 
% As is shown by Fig.~1, the projection of $(c_{ij^*}, r_{ij^*})$ is largest and it satisfies 
% $$\ell_{ij^*} \le \alpha \le \ell_{ij^*-1}.$$

Let $(c_{ij}, r_{ij})$ be a point in rectangular coordinates. The value of $r_{ij} - \alpha c_{ij}$ can be regarded as the projection of $(c_{ij}, r_{ij})$ on $(-\alpha, 1)$. 
Let $j^*$ be the solution of Lagrangian duality methods in Eq.~\eqref{eq:dual-solution}, i.e., $j^* = \arg \max_j r_{ij} - \alpha c_{ij}$. Based on Assumption~\ref{assumption:law}, we can prove that it satisfies
$$\ell_{ij^*} \le \alpha \le \ell_{ij^*-1}.$$

\noindent Therefore, an equivalent optimization algorithm is obtained in Algorithm~\ref{alg:MTBAP}. Due to the space limit, the formal proof of Theorem~\ref{theorem:MTBAP-equivalent} can be found in Appendix~C~\cite{Zhou2022DHCL}.

\begin{algorithm}[tb]
\caption{An Equivalent Algorithm for MTBAP}
\label{alg:MTBAP}
\textbf{Input}: $(\ell_{ij}, B)$\\
\textbf{Output}: the solution to MTBAP
\begin{algorithmic}[1] %[1] enables line numbers
\STATE $\alpha_l = 0, \alpha_r = \max_{ij} \ell_{ij}, z_{ij} = 0 \ \text{for} \  \forall i,j$.
\STATE /*The same binary search as Lagrangian duality.*/
\WHILE {$B - \sum_{ij} z_{ij} c_{ij} > \epsilon$}
\STATE $z_{ij} = 0 \ \text{for} \  \forall i,j.$
\STATE $\alpha = (\alpha_l + \alpha_r)/2.$
% \STATE /*Equivalent to $j^* = \arg\max_j r_{ij} - \alpha c_{ij}$*/
\FOR[Equivalent to $j^* = \arg\max_j r_{ij} - \alpha c_{ij}$]{$\forall i$}
\STATE $\ell_{i0} = +\infty, \ \ell_{iL} = 0.$
\STATE find $j^*$ satisfying $\ell_{ij^*} \le \alpha <\ell_{ij^*-1}.$
\STATE $z_{ij^*} = 1.$
\ENDFOR
\IF {$B - \sum_{ij} z_{ij} c_{ij} >= 0$}
\STATE $\alpha_r = \alpha.$
\ELSE
\STATE $\alpha_l = \alpha.$
\ENDIF
\ENDWHILE
\RETURN $z_{ij}$.

\end{algorithmic}
\end{algorithm}

\begin{theorem} \label{theorem:MTBAP-equivalent}
Algorithm~\ref{alg:MTBAP} is equivalent to the Lagrangian duality method.
\end{theorem}

Notice that the factor $\ell_{ij}$ in Algorithm~\ref{alg:MTBAP} can be taken as a decision factor, which is irrelevant to the 
Lagrangian multiplier $\alpha$. Therefore, the value of $\ell_{ij}$ (or implicitly, the rank of $\ell_{ij}$ for all $i,j$) can be taken as the learning objective, which avoids the difficulty of directly predicting $r_{ij} - \alpha c_{ij}$ for all possible $\alpha$ in the Lagrangian duality method.

\subsection{Direct Prediction Model}

Based on the above analysis, we propose a novel method for a direct prediction of $\ell_{ij}$ in this subsection. Suppose that there is a data set $(x_i, t_i, y^c_i, y^r_i)$ of size $N$ collected from RCT, and denote by $N_j$ the count of samples receiving treatment $j$. Similarly, the range of $\ell_{ij}$ is limited to $(0,1)$ by scaling and truncating operations for $Y^r$ or $Y^c$ to reduce the risk of overfitting. Let $s_{ij} = \hbar(x_i, j)$ be the prediction of the rank of $\ell_{ij}$, where $\hbar(\cdot)$ is any machine learning model. Hence, minimize the loss function~\eqref{MTBAP-loss} and we can get the unbiased estimation of $\ell_{ij}$. The detailed analysis is shown by Theorem~\ref{theorem:MTBAP-loss} and its proof in Appendix~D~\cite{Zhou2022DHCL}.

\begin{align}
\min& L(s) = - [\sum_{j>1} \sum_{i|t_i = j} \frac{1}{N_j} (q_{ij-1} y^r_i - q^2_{ij-1} y^c_i) - \nonumber\\ &\ \ \ \ \ \ \ \ \ \ \ \  \sum_{j<L} \sum_{i|t_i = j} \frac{1}{N_j} (q_{ij} y^r_i - q^2_{ij} y^c_i)] \label{MTBAP-loss}\\
s.t.& \ \ \ \ \ \ \ \ \ \ \ \ \ \ \ \ q_{ij} = \sigma(s_{ij}), \forall i, j& \nonumber\\
& \ \ \ \ \ \ \ \ \ \ \ \ \ \ \ \ s_{ij} = \hbar(x_i, j) \in \mathbb{R}, \forall i, j.& \nonumber
\end{align}

\begin{theorem}\label{theorem:MTBAP-loss}
When the loss function~\eqref{MTBAP-loss} converges, $s_{ij}$ can be used to rank $\ell_{ij}$ and $q_{ij}$ can be used to obtain the unbiased estimation of $\ell_{ij}$. 
\end{theorem}

\subsection{An Evaluation Metric}

Although AUUC (Area Under Uplift Curve) and AUCC (Area Under Cost Curve)~\cite{Du2019Improve} have been developed to evaluate the ranking performance of uplift models without/with treatment cost respectively, there is no evaluation metric for the estimation of marginal utilities ($\ell_{ij}$) under different treatments. The latter is directly related to the business objective of MTBAP. Therefore, we propose a novel evaluation metric for such a purpose, which is called as MT-AUCC (Area Under Cost Curve for Multiple Treatments) in this paper. 

\textbf{Cost Curve for Multiple Treatments.} Suppose that there is a model $\mathcal{M} = f(x_i, j)$ to predict the value (or the rank) of $\ell_{ij}$. Firstly, we obtain two new quintuple sets based on $\mathcal{M}$.
\begin{itemize}
\item For each sample $(x_i, t_i, y^c_i, y^r_i)$ with $t_i < L$, use model $\mathcal{M}$ to obtain a score $S_i = f(x_i, t_i)$ and get a new quintuple set $\widetilde{T}= \{(x_i, t_i, \alpha_i y^c_i, \alpha_i y^r_i, S_i) | t_i < L,  \alpha_i = \frac{N}{N_{t_i}}\}$;
\item for each sample $(x_i, t_i, y^c_i, y^r_i)$ with $t_i > 1$, use this mode again to obtain $S_i = f(x_i, t_i - 1)$ and get a quintuple set $\widetilde{C} = \{(x_i, t_i, \alpha_i y^c_i, \alpha_i y^r_i, S_i)| t_i > 1, \alpha_i = \frac{N}{N_{t_i}}\}$.
\end{itemize}
Notice that the weight $\alpha_i$ is used to balance the count of samples under different treatments. Next regard $\widetilde{T}$ and $\widetilde{C}$ as the new treatment group and control group respectively. Sort all the quintuples in $\widetilde{T}$ and $\widetilde{C}$ in descending order of $S_i$. For the top $k$ quintuples in this sorted list, denote by $\widetilde{T}(k)$ (or $\widetilde{C}(k)$) the quintuples that belong to $\widetilde{T}$ (or $\widetilde{C}$). Therefore, the incremental cost $\Delta Y^c(k)$ and reward $\Delta Y^r(k)$ can be calculated for each point $1 \le k \le |\widetilde{T}| + |\widetilde{C}|$ by Eq.~\eqref{Eq:MT-AUCC-Y}.
\begin{align}
\label{Eq:MT-AUCC-Y}
&\Delta Y^*(k) = \frac{k}{|\widetilde{T}| + |\widetilde{C}|} \left( \frac{\sum_{i \in \widetilde{T}(k)} \alpha_i y_i^*}{| \widetilde{T}(k)|} - \frac{\sum_{i \in \widetilde{C}(k)} \alpha_i y_i^*}{| \widetilde{C}(k)|} \right) , \nonumber \\ 
&\text{where} \ \ * \in \{r,c\} 
\end{align} 
As is shown in Fig.~\ref{fig:MT-AUCC}, take the tuple $(\Delta Y^c(k), \Delta Y^r(k))$ as the coordinates and we can get a cost curve. 

Denote by $\Delta Y_c$ and $\Delta Y_r$ the average incremental cost and reward of all the samples in $\widetilde{T}$ and $\widetilde{C}$, which satisfies $\Delta Y_* = \Delta Y^*(|\widetilde{T}| + |\widetilde{C}|)$. For convenience of calculations, we can also use some points with $p$ percent of $\Delta Y_c$ to draw the cost curve. In addition, the X and Y axis in this curve can be normalized by being divided by $\Delta Y_c$ and $\Delta Y_r$ respectively.

\textbf{Area Under Cost Curve for Multiple Treatments (MT-AUCC).} Similar to AUUC and AUCC, the area under this cost curve can be regarded as an evaluation metric. As is shown by Fig.~\ref{fig:MT-AUCC}, denote by $A_M$ and $A_R$ the area under a model curve and a random benchmark curve respectively. In order to bound the result within $(0, 1]$, MT-AUCC of this model is defined as $A_M/2A_R$.

\begin{figure}[htbp]
\centering  %图片全局居中
\includegraphics[width=0.35\textwidth]{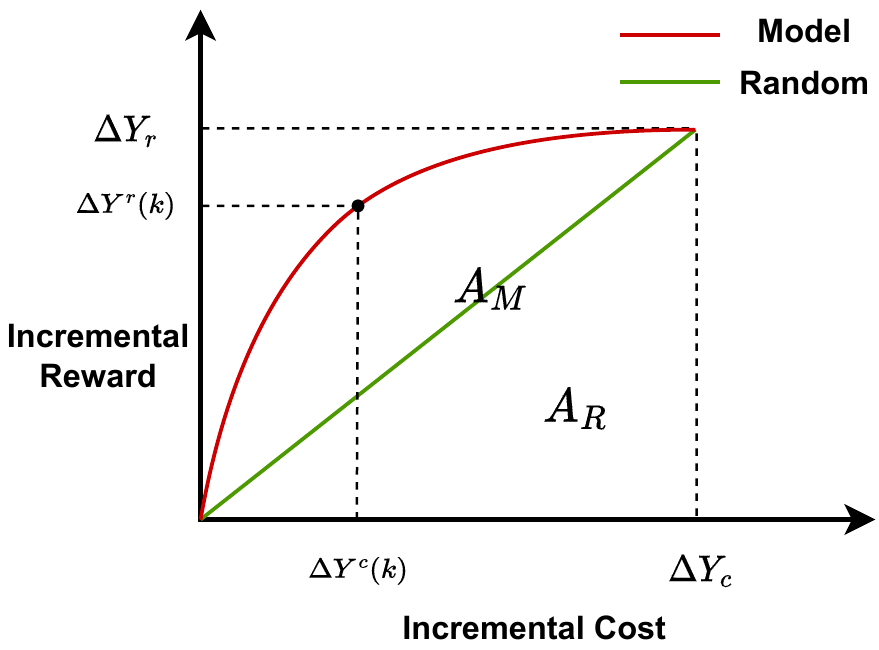}
\caption{MT-AUCC}
\label{fig:MT-AUCC}
\end{figure}

\section{Evaluation}

In this section, we will conduct large-scale offline and online numerical experiments to validate the performance of our models and algorithms. 
% The offline simulation is based on three kinds of dataset: an open real-world dataset, a synthetic dataset and a marketing dataset in an online food delivery platform. Our models has been deployed for online service in this platform and a continuous A/B test is used for online evaluation.

\begin{figure*}[htbp]
\centering  %图片全局居中
\subfigure[AUUC (CRITEO-UPLIFT v2)]{
\label{fig:evaluation-1}
\includegraphics[width=0.24\textwidth]{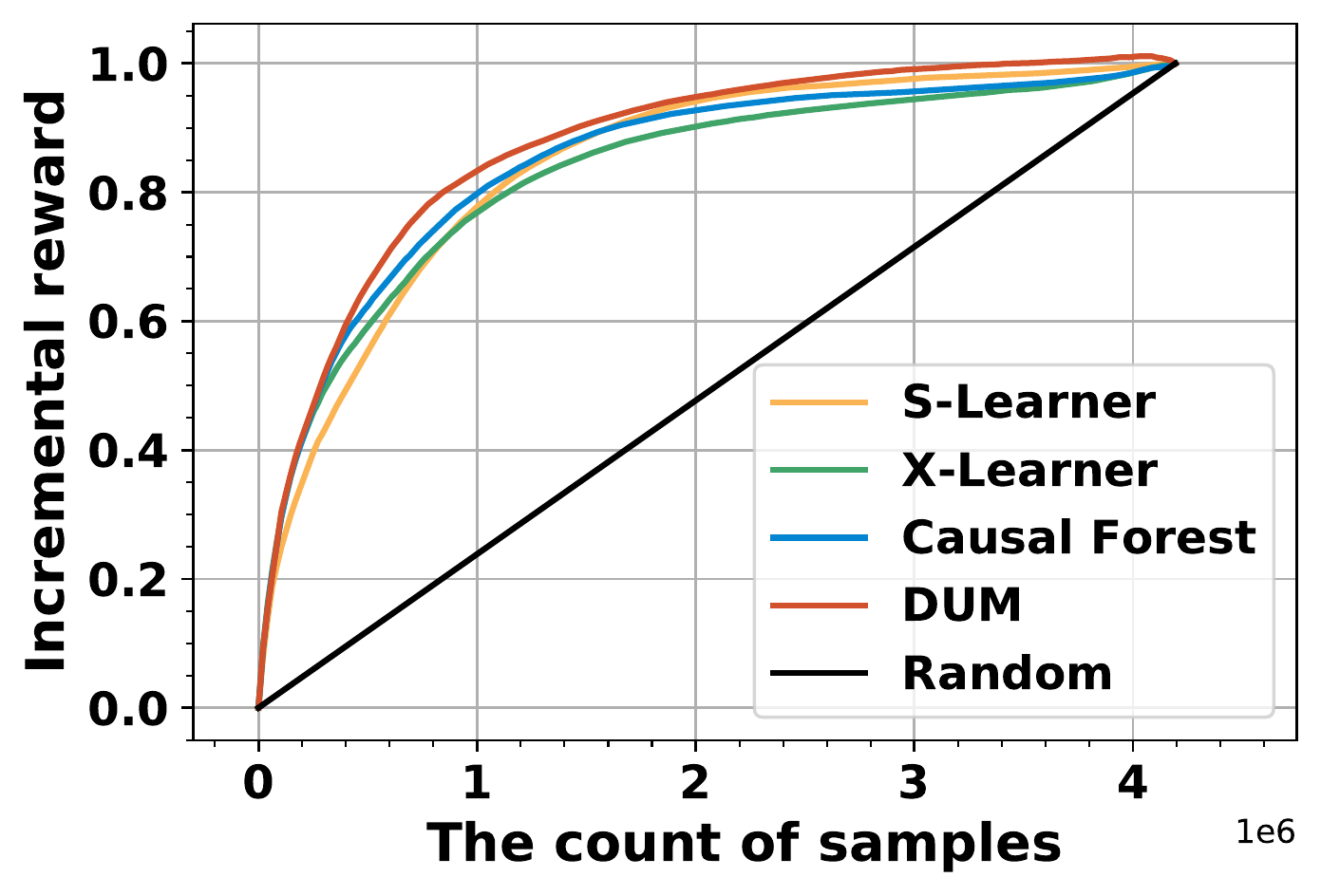}}
\subfigure[AUUC (Marketing data)]{
\label{fig:evaluation-2}
\includegraphics[width=0.24\textwidth]{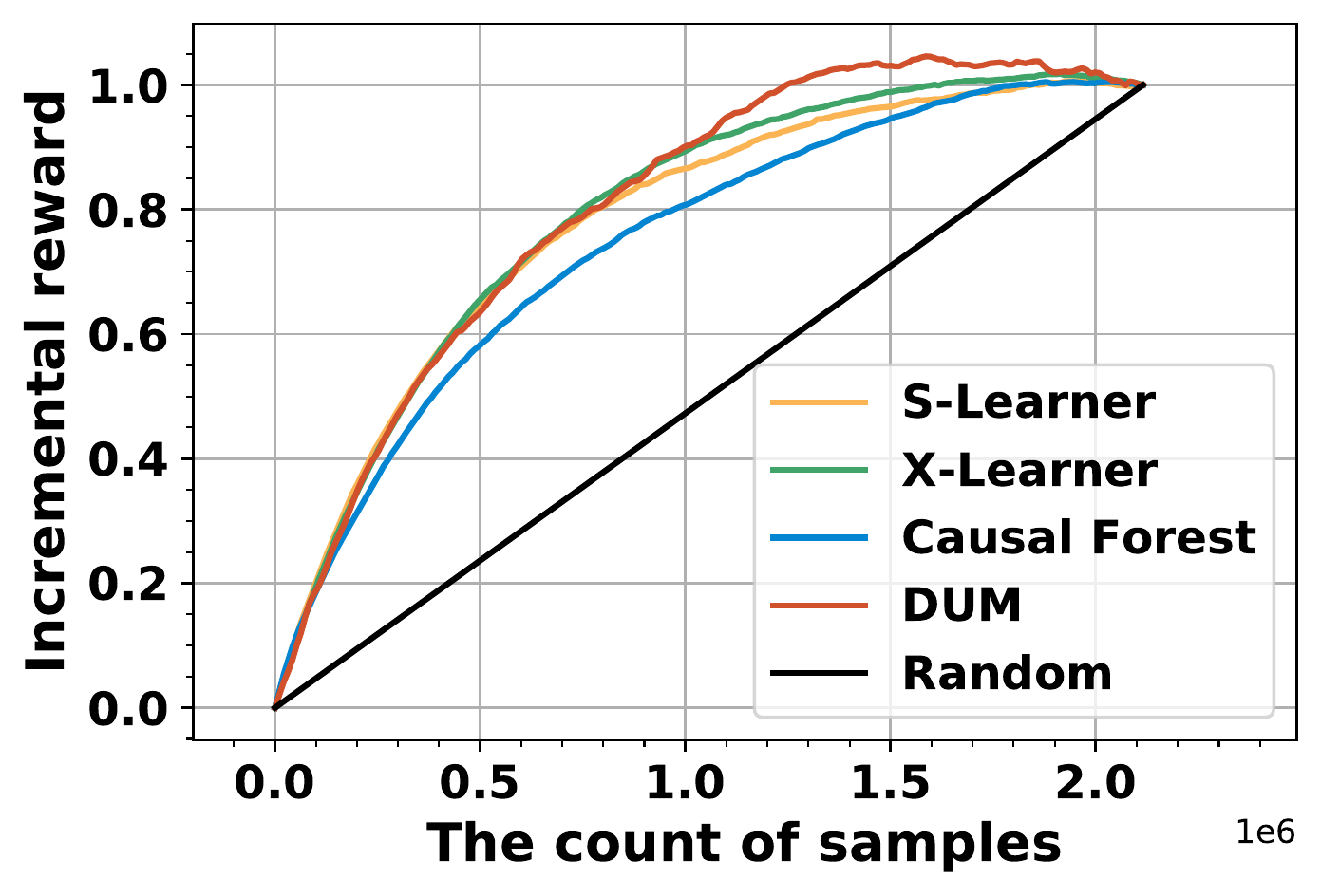}}
\subfigure[AUCC (CRITEO-UPLIFT v2)]{
\label{fig:evaluation-3}
\includegraphics[width=0.24\textwidth]{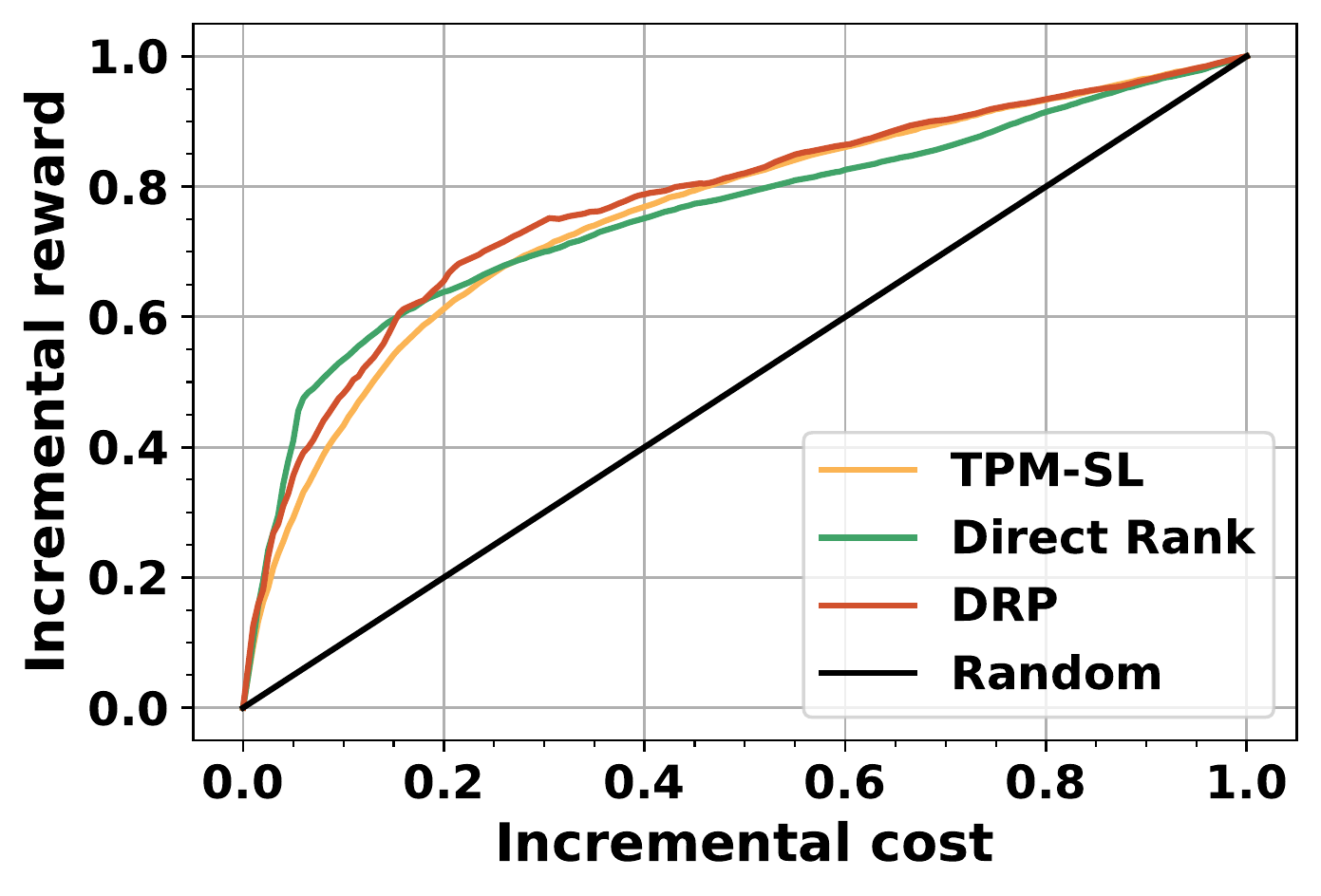}}
\subfigure[AUCC (Marketing data)]{
\label{fig:evaluation-4}
\includegraphics[width=0.24\textwidth]{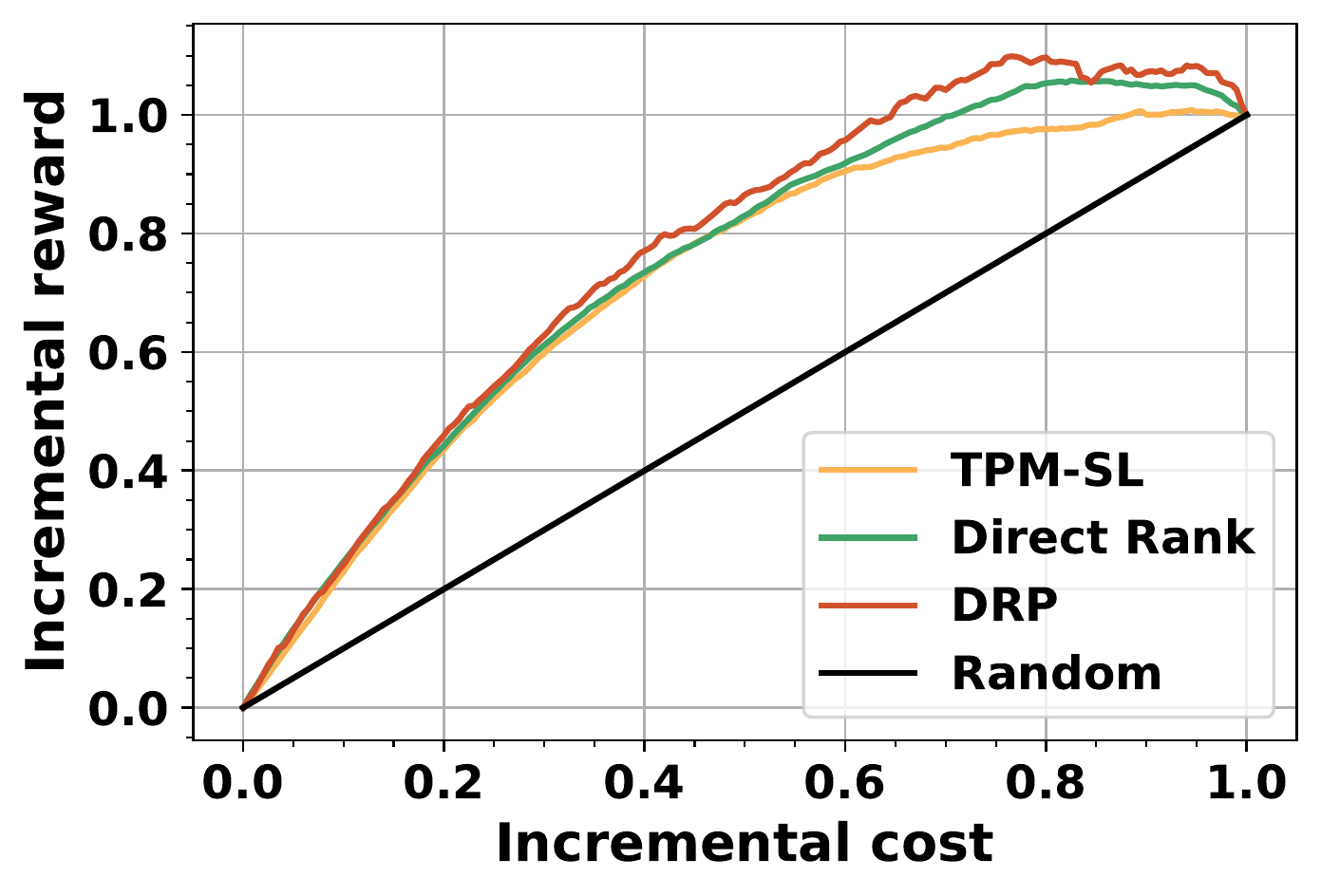}}

\subfigure[MT-AUCC (Marketing data)]{
\label{fig:evaluation-5}
\includegraphics[width=0.24\textwidth]{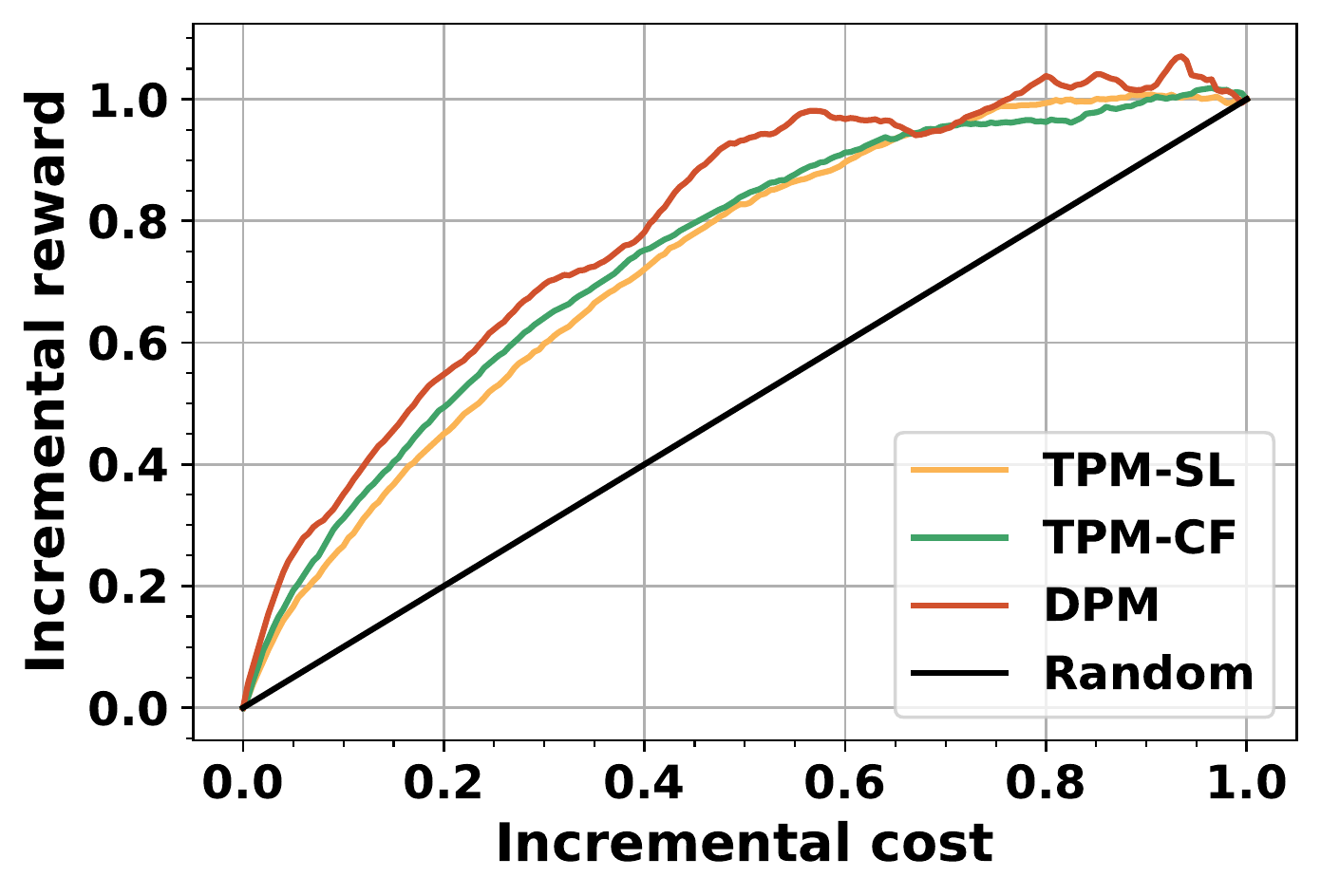}}
\subfigure[EOM (Marketing data)]{
\label{fig:evaluation-6}
\includegraphics[width=0.233\textwidth]{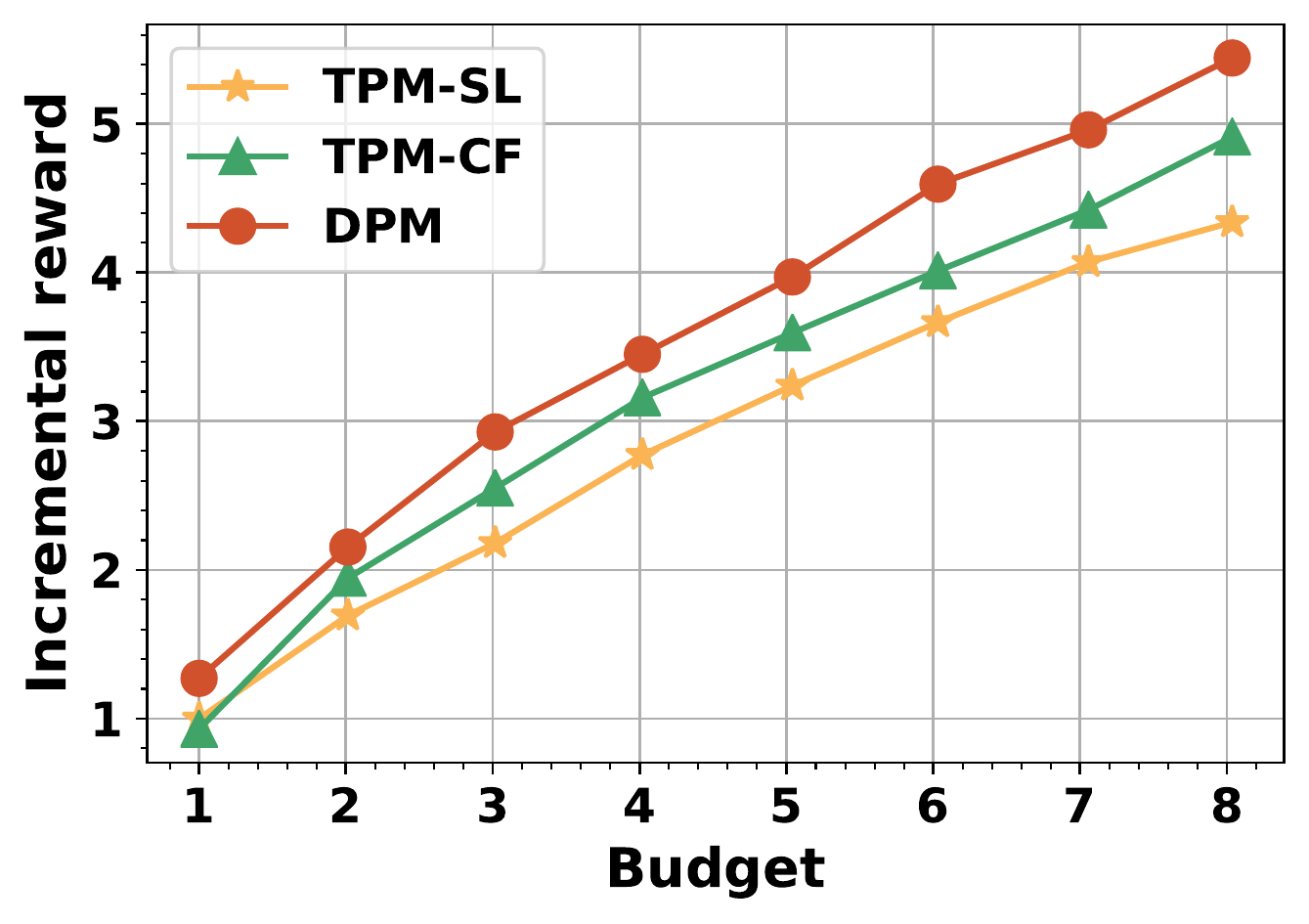}}
\subfigure[Orders (Online AB Test)]{
\label{fig:evaluation-7}
\includegraphics[width=0.24\textwidth]{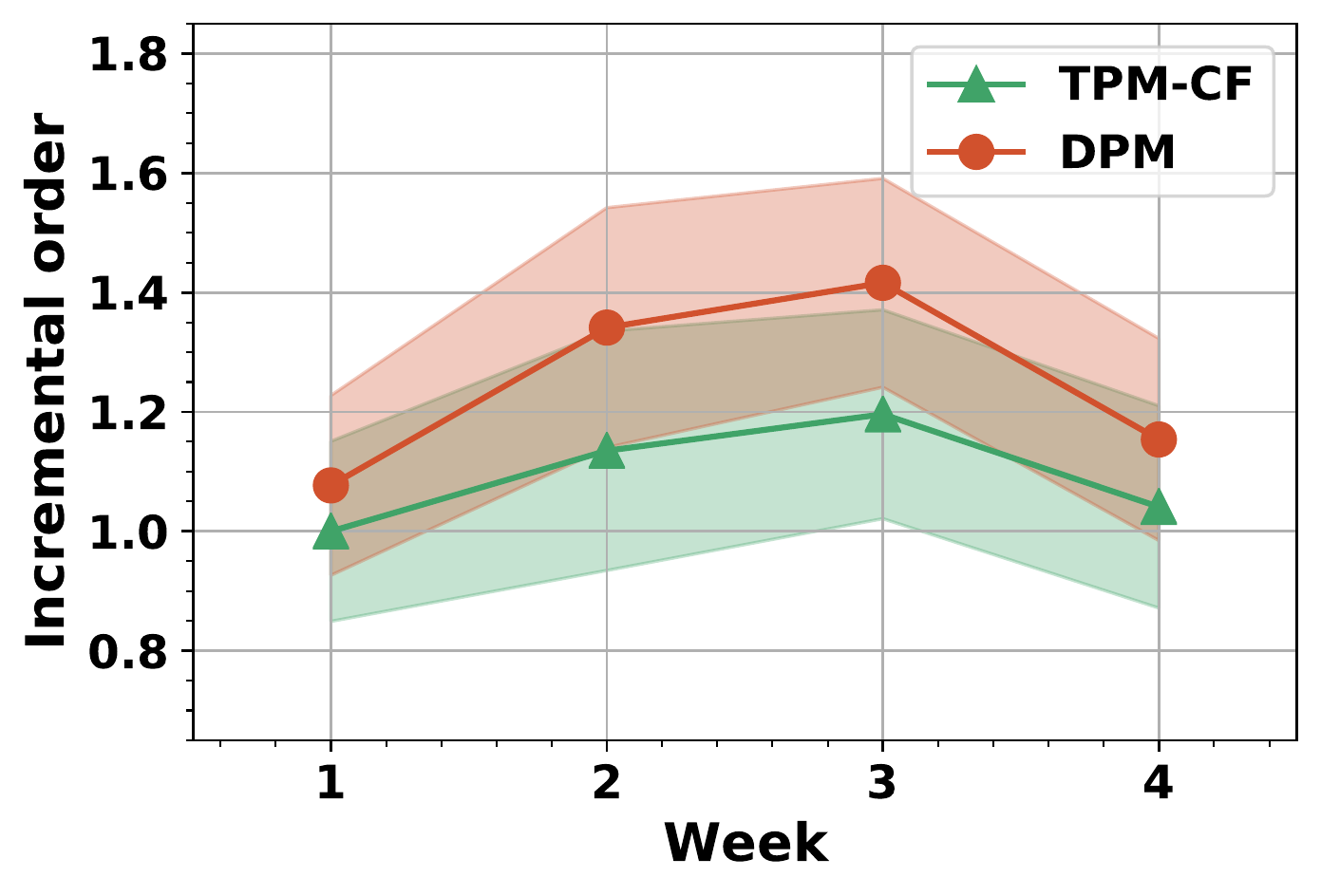}}
\subfigure[GMV (Online AB Test)]{
\label{fig:evaluation-8}
\includegraphics[width=0.24\textwidth]{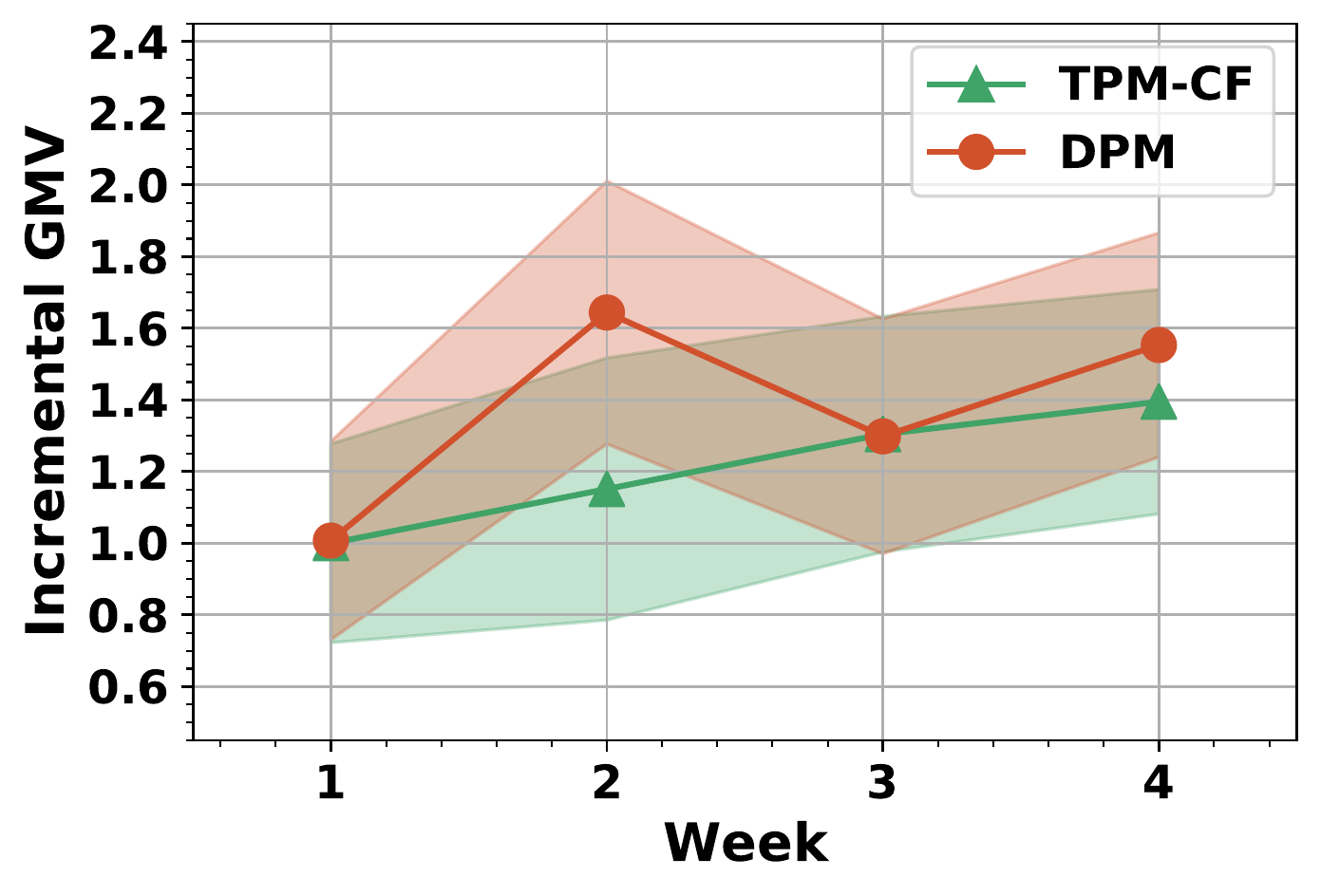}}
\caption{Evaluation Results}
\label{fig:evaluation}
\end{figure*}

\subsection{Offline Simulation}

\subsubsection{Dataset.} Two types of datasets are provided in this paper: an open real-world dataset and a marketing dataset collected from an online food delivery platform.

\begin{itemize}
    \item CRITEO-UPLIFT v2. This dataset is provided by the AdTech company Criteo in the AdKDD’18 workshop~\cite{Diemert2018}. The data is collected from a random control trial (RCT) that prevents a random part of users from being targeted by advertising. It contains 12 features, 1 binary treatment indicator and 2 response labels (visit/conversion). The (incremental) visit is regarded as the predictive objective for (cost-unaware) uplift modeling.  In order to compare the performance of different models to predict ROI of individuals, we take the visit label as the cost and the conversion label as the reward. The whole dataset contains 13.9 million samples, and is randomly partitioned into two parts for train (70\%) and test (30\%), respectively.
    \item Marketing data. Money Off is a common marketing campaign in Meituan, an online food delivery platform. We conduct a four-week RCT in this platform where online shops will offer a random discount every day. Notice that the discount of a shop is the same for all the users to prevent price discrimination but may randomly change in different days, and different shops may offer different discounts.  The data in the first two weeks is used for train and the others for test. The discount $T \in \{0, 1, \ldots, 6\}$ is taken as the treatment, where $T = t$ means $\$t$ cash off for each order whose price meets a given threshold. This dataset contains 75 features, 1 treatment label and 2 response labels (daily cost/orders). For the binary treatment assignment problem, we take the samples with $T = 0$ as the control group and the samples with $T > 0$ as the treatment group. For the budget allocation problem with multiple treatments, the budget refers to the whole cost of all the shops and different discounts represent different treatments. This dataset contains 4.1 million samples.
\end{itemize}

\subsubsection{Evaluation Metric.}

Multiple evaluation metrics are provided for offline evaluation in this experiment. 

\begin{itemize}
    \item AUUC (Area under Uplift Curve). A common metric~to evaluate uplift models \cite{Rzepakowski020Decision}. In this experiment, the auuc score is computed by using CausalML packages~\cite{chen2020causalml}.
    \item AUCC (Area under Cost Curve). A similar metric to AUUC, but designed for evaluating the performance to rank ROI of individuals~\cite{Du2019Improve}.
    \item MT-AUCC. It is proposed in this paper and used to evaluate the performance of models to rank marginal utilities of different individuals under different treatments.
    \item EOM (Expected Outcome Metric). Based on RCT data, the expected outcome (reward/cost) can be obtained for arbitrary policy by using the computing methods in~\cite{Ai2022LBCF,Zhao2017Uplift}.
\end{itemize}

\subsubsection{Benchmark.} 

For each problem considered in this paper, multiple different models/algorithms are implemented and taken as the benchmarks.

\begin{itemize}
    \item Cost-unaware binary treatment assignment problem
        \begin{itemize}
            \item S-Learner. A single model predicting the response of individuals with/without the treatment. The CATE is computed by $E(Y|X, T=1) - E(Y|X, T=0)$.
            \item X-Learner. A meta-learner approach proposed in \cite{Kunzel2019Metalearners}.
            \item Causal Forest. An uplift model proposed in \cite{Athey2019Generalized}. It is implemented here based on EconML packages~\cite{econml}.
            \item DUM. The direct uplift modeling method in this paper.
        \end{itemize}
    \item Cost-aware binary treatment assignment problem
        \begin{itemize}
            \item TPM-SL. The two-phase method which uses two S-Learner models to predict the incremental revenue and cost, respectively. Predict ROI of individuals by computing the ratio between these two models. 
            % \item TPM-CF. Instead of s-learner, it takes causal forests as the uplift model and computes the ratio. 
            \item Direct Rank. Similar to our work, a loss function is designed for ranking ROI of individuals in this model~\cite{Du2019Improve}. However, we prove that it cannot achieve the correct rank when the loss converges in Appendix~E~\cite{Zhou2022DHCL}.
            % \item Instrumental Forest. Taking the original treatment as an instrumental variable and the cost as a new treatment, ROI of an individual can be obtained by computing the (conditional) local average treatment effect (LATE)~\cite{Sun2021Treatment}. In this experiment, we use the instrumental forest in EconML to get the estimation of LATE.
            \item DRP. The direct ROI prediction model in this paper.
        \end{itemize}
    \item Budget allocation problem with multiple treatments
        \begin{itemize}
            \item TPM-SL. The two-phase method mentioned in many existing works~\cite{Ai2022LBCF,Zhao2019Unified}. In the first stage, we use a S-Learner model to predict the response (reward/cost) of individuals under different treatments. In the second stage, the Lagrangian duality algorithm is developed to compute the approximately optimal solution.
            \item TPM-CF. Instead of S-Learner, we use Causal Forests to predict the incremental response. It is implemented based on generalized random forests (GRF) in EconML packages~\cite{econml}, which can also support multiple treatments.
            \item DPM. The approach in this paper that combines the direct prediction of marginal utilities and Algorithm~\ref{alg:MTBAP}.
        \end{itemize}
\end{itemize}

The hyperparameters in these algorithms are obtained based on grid search and each data point in the experimental results is computed by running the programs for 20 times.

\subsubsection{Experimental Results.} For the cost-unaware binary treatment assignment problem, Fig.~\ref{fig:evaluation-1}-\ref{fig:evaluation-2} presents the comparison of four uplift models. First of all, our model DUM performs best in both CRITEO-UPLIFT v2 and Marketing data. As the common baseline, the result of S-Learner is not too bad. It is near to our algorithm DUM in both two datasets. For comparison, X-Learner and Causal Forest is not always superior to S-Learner. The former is worse in CRITEO-UPLIFT v2 and the latter is inferior to S-Learner in Marketing data. The detailed results can be found in Table~1 in Appendix~F~\cite{Zhou2022DHCL}.
% Although X-Learner performs better than S-Learner in Marketing data, Table~\ref{tab:evaluation} shows that its result is unstable and has high variance. Compared with S-Learner and X-Learner, Causal Forest is more stable but fails to achieve good performance in both two datasets.

Due to the robustness of S-Learner, it is still taken as the base model to predict ROI of individuals. As is shown by Fig.~\ref{fig:evaluation-3}-\ref{fig:evaluation-4}, TPM-SL cannot perform well especially in Marketing data at this time. 
% As the resource allocation problem gets more complicated, we believe that two-phase methods will perform worse because of the enlargement of model errors. 
% Direct Rank also makes a direct prediction, but it uses an incorrect loss function. Hence, it cannot converge to a stable extreme point and is inferior to our model DRP. 
The incorrect loss function of Direct Rank causes that it cannot converge to a stable extreme point and is inferior to our model DRP.
In Appendix~E~\cite{Zhou2022DHCL}, we will present the detailed analysis for the convergence of Direct Rank. %and show that its result will be worse after the convergence. 
Compared with TPM-SL and Direct Rank, our model DRP always performs best and achieves significant improvement. 
% As is shown in Table~\ref{tab:evaluation}, it is worth noting that the result of DRP is very stable and has very low variance. This is because the loss function of DRP is convex and it can always converge a stable extreme point.

Fig.~\ref{fig:evaluation-5}-\ref{fig:evaluation-6} shows the results of different models and algorithms to solve the budget allocation problem with multiple treatments. Since the tree-based uplift models were often used in many existing works~\cite{Ai2022LBCF,Zhao2017Uplift} to deal with this problem, we also take TPM-CF as the baseline. Our approach DPM significantly outperforms TPM-SL and TPM-CF in MT-AUCC, which indicates that DPM is better at ranking marginal utilities. We also use EOM to test the incremental reward of different approaches when given different budget in Fig.~\ref{fig:evaluation-6}. In order to protect the data privacy of this platform, the budget and reward have been normalized. In spite of this, it is still clear that our approach DPM can always help the platform to obtain much more reward under different budget.

As is shown by Table~1 in Appendix~F~\cite{Zhou2022DHCL}, all the models proposed in this paper are more stable and have lower variance than other existing works. This is because our models can make a direct prediction for the final objective, and always converge to a stable extreme point.

\subsection{Online A/B Test}

\subsubsection{Setups.} We deploy our algorithm (DPM) to support the Money Off campaign in Meituan (a food delivery platform), and conduct an online AB test for four weeks. There are 310k total shops in this experiment and they are randomly partitioned into three groups, named G-DPM, G-TPM and G-Control respectively. 
The discount $T \in \{0,1,\ldots,6\}$ is taken as the treatment and assigning a shop with treatment $T = t$ means $\$t$ cash off for each order whose price meets a given threshold. Given a limited budget, the objective is to decide the discount every day for each shop so as to maximize the total number of orders and GMV (Gross Merchandise Value) in this platform. Algorithm DPM and TPM-CF are deployed in the experiment groups named G-DPM and G-TPM, respectively. The group G-Control is taken as the control group and does not offer any discount. These groups will be randomly broken up every week. Therefore, the AB experiment is repeated four times and the period of each time is one week.

\subsubsection{Results.} Fig.~\ref{fig:evaluation-5}-\ref{fig:evaluation-6} shows the incremental orders and GMV relative to G-Control in each week. To protect data privacy, all the data points have been normalized that are divided by the incremental orders or GMV of TPM-SL in the first week. The shadow area in Fig.~\ref{fig:evaluation-5}-\ref{fig:evaluation-6} represents the confidence interval with a confidence level of 0.95, which is calculated by student’s t-test. Compared with TPM-CF, our approach DPM always performs better in incremental orders and is not inferior to it in incremental GMV in each week. To sum up, DPM achieves a significant growth by 14.3\% in incremental orders and 13.6\% in GMV on average.

\section{Conclusion}

% In this paper, we pointed out some major drawbacks of two-phase methods to solve resource allocation problems in user marketing, and proposed that the prediction objective of the model should depend on the specific combinatorial optimization algorithm.
% In order to guarantee the consistency between machine learning (ML) and operation research (OR), we proposed that the learning objective of predictive models should depend on the specific combinatorial optimization algorithm in this paper. 
% This idea was applied to solve three crucial problems in marketing and presented great advantages both theoretically and practically. 
% Specifically, for each problem, the key decision factor with direct correlation to the optimization results was extracted and taken as the learning objective, and a global loss function was developed to build the direct prediction model.
% Large-scale offline simulations and online AB tests validated the effectiveness of our approach. 

In this paper, we proposed a novel approach for solving resource allocation problems based on the decision factor. Taking it as the learning objective can avoid alternative mathematical operations performed on the prediction results. 
This idea was applied to solve two crucial problems in marketing and presented great advantages both theoretically and practically. 
% Specifically, for each problem, the decision factor was extracted from the combinatorial optimization algorithm, and direct heterogeneous causal learning for it was conducted based on a global loss function.
Large-scale offline simulations and online AB tests validated the effectiveness of our approach. 

Our future work will focus on the application of this approach in more complex marketing scenarios. For example, multiple marketing campaigns may be conducted at the same time and interact with each other. Therefore, deriving the decision factor and conducting direct heterogeneous causal learning in this situation are more challenging.

% \section{Acknowledgments}
% AAAI is especially grateful to Peter Patel Schneider for his work in implementing the original aaai.sty file, liberally using the ideas of other style hackers, including Barbara Beeton. We also acknowledge with thanks the work of George Ferguson for his guide to using the style and BibTeX files --- which has been incorporated into this document --- and Hans Guesgen, who provided several timely modifications, as well as the many others who have, from time to time, sent in suggestions on improvements to the AAAI style. We are especially grateful to Francisco Cruz, Marc Pujol-Gonzalez, and Mico Loretan for the improvements to the Bib\TeX{} and \LaTeX{} files made in 2020.

% The preparation of the \LaTeX{} and Bib\TeX{} files that implement these instructions was supported by Schlumberger Palo Alto Research, AT\&T Bell Laboratories, Morgan Kaufmann Publishers, The Live Oak Press, LLC, and AAAI Press. Bibliography style changes were added by Sunil Issar. \verb+\+pubnote was added by J. Scott Penberthy. George Ferguson added support for printing the AAAI copyright slug. Additional changes to aaai23.sty and aaai23.bst have been made by Francisco Cruz, Marc Pujol-Gonzalez, and Mico Loretan.

% \bigskip
% \noindent Thank you for reading these instructions carefully. We look forward to receiving your electronic files!

\bibliography{aaai23}

\appendix

\setcounter{theorem}{0}

\section{Appendix A. Proof of Theorem~\ref{theorem:DUM}}

\begin{theorem} %\label{theorem:DUM}
When the loss function~(2) converges, $s_i$ can be used to rank $\tau^r(x_i)$ and $q_i = \frac{\tau^r(x_i)}{\sum_i \tau^r(x_i)}$ can be used to obtain an unbiased estimation of $\tau^r(x_i)$. 
\end{theorem}

\begin{proof}
Denote $q_i$ by $\kappa(x_i)$ in the loss function~(2). 
Therefore, we have 
\begin{align*}
&\frac{1}{N_1}\sum_{i|t_i=1} y^r_i \ln q_i \\
=& E[Y^r_i \ln \kappa(X_i) | T_i = 1]\\
=& E[Y^r_i(1) \ln \kappa(X_i) | T_i = 1] \ \ \ \ (\text{SUTVA}\footnotemark)\\
=& E[Y^r_i(1) \ln \kappa(X_i)] \ \ \ \ ((Y^r(0), Y^r(1)) \perp T).
\end{align*}
The third equation holds based on the property of random control trials (RCT).
Hence, $L(s)$ can be rewritten as the following.
\begin{align*}
L(s) =& -(E[Y^r_i(1) \ln \kappa(X_i)] - E[Y_i(0) \ln \kappa(X_i)]) \\
=& -E[(Y^r_i(1) - Y^r_i(0)) \ln \kappa(X_i)] \\
=& - E_X[E[(Y^r(1)-Y^r(0)) \ln \kappa(X) | X]] \\
=& - E_X[E[\tau^r(X) \ln \kappa(X) | X]] \\
=& - E[\tau^r(X_i) \ln \kappa(X_i)] \\ 
=& - \frac{1}{N} \sum_i \tau^r(x_i) \ln q_i.
\end{align*}
When the loss converges, we have
\begin{align*}
\frac{\partial L(s)}{\partial s_j} =& - \frac{1}{N} \sum_i \frac{\tau^r(x_i)}{q_i} \frac{\partial q_i}{\partial s_j} \\
=& - \frac{1}{N} [\tau^r(x_j) - q_j \sum_i \tau^r(x_i)] = 0.
%=& 0.
\end{align*}
It implies that $q_j = \frac{\tau^r(x_j)}{\sum_i \tau^r(x_i)}$ holds for $\forall j$. According to the loss function~(2), $q_i < q_j$ can be derived from $s_i < s_j$, where the former means $\tau^r(x_i) < \tau^r(x_j)$. 

\footnotetext{This is an important assumption in causal inference~\cite{sekhon2008neyman}. It requires that the potential outcome on one unit should be unaffected by the particular assignment of treatments to the other units.}

\end{proof}

\section{Appendix B. Proof of Theorem~\ref{theorem:DRP}}

\begin{theorem} %\label{theorem:DRP}
The loss function~(3) is convex , and $s_i$ can be used to rank ROI and $q_i = \frac{\tau^r(x_i)}{\tau^c(x_i)}$ is an unbiased estimation of ROI of individual $i$ when the loss converges. 
\end{theorem}

\begin{proof}
Similarly to the proof of Theorem~\ref{theorem:DUM}, the loss function~(3) can be rewritten as the following equation. 
% $$L(s) = - \frac{1}{N}\sum_{i} \tau^r(x_i) \ln \frac{q_i}{1-q_i} + \tau^c(x_i) \ln (1 - q_i).$$
\begin{align*}
L(s) &= - [\frac{1}{N_1}\sum_{i|t_i=1} y^r_i \ln \frac{q_i}{1-q_i} -  \frac{1}{N_0}\sum_{i|t_i=0} y^r_i \ln \frac{q_i}{1-q_i}]  - \\ &\ \ \ \ \  [\frac{1}{N_1}\sum_{i|t_i=1} y^c_i \ln (1 - q_i) -  \frac{1}{N_0}\sum_{i|t_i=0} y^c_i \ln (1 - q_i)] \\
&= - \frac{1}{N}\sum_{i} [\tau^r(x_i) \ln \frac{q_i}{1-q_i} + \tau^c(x_i) \ln (1 - q_i)].
\end{align*}
Therefore, we have
\begin{align*}
\frac{\partial L(s)}{\partial s_i} &= - \frac{1}{N} (\frac{\tau^r(x_i)}{q_i(1-q_i)} - \frac{\tau^c(x_i)}{1 - q_i}) \frac{\partial q_i}{\partial s_i} \\
&= - \frac{1}{N} (\tau^r(x_i) - q_i \tau^c(x_i))
\end{align*}
and
\begin{align*}
\frac{\partial^2 L(s)}{{\partial s_i}^2} &= \frac{\tau^c(x_i)}{N}  \frac{\partial q_i}{\partial s_i} \\
&= \frac{q_i (1 -q_i) \tau^c(x_i)}{N}.
\end{align*}
Due to $q_i \in (0,1)$ and $\tau^c(x_i) > 0$, we get $\frac{\partial^2 L(s)}{{\partial s_i}^2} > 0$. 
In addition, it is easy to verify that $\frac{\partial^2 L(s)}{\partial s_i \partial s_j} = 0$ also holds. 
Based on the above, the loss function $L(s)$ is convex. When the loss converges, we can get $\frac{\partial L(s)}{\partial s_i} = 0$ for $\forall i$, which induces that $q_i = \frac{\tau^r(x_i)}{\tau^c(x_i)}$ holds. Since $\sigma(\cdot)$ is a monotone increasing function, the sort of $s_i$ can represent the rank of ROI of individual $i$.

\end{proof}

\section{Appendix C. Proof of Theorem~\ref{theorem:MTBAP-equivalent}}

\begin{theorem} %\label{theorem:MTBAP-equivalent}
Algorithm~2 is equivalent to the Lagrangian duality method.
\end{theorem}

\begin{proof}
In order to prove the equivalence, we only need to prove that the following properties hold.
\begin{enumerate}[Property 1]
\item For $\forall i$, if $j^* = \arg \max_j r_{ij} - \alpha c_{ij}$ holds, then $\ell_{ij^*} \le \alpha \le \ell_{ij^*-1}$ holds.
\item For $\forall i$, if $\ell_{ij^*} \le \alpha \le \ell_{ij^*-1}$ holds, then $j^* = \arg \max_j r_{ij} - \alpha c_{ij}$ holds.
\end{enumerate}

For Property~1, if $j^* = \arg \max_j r_{ij} - \alpha c_{ij}$ holds, the following equations also holds.
\begin{align*}
r_{ij^*-1} - \alpha c_{ij^*-1} &\le r_{ij^*} - \alpha c_{ij^*}, \\
r_{ij^*+1} - \alpha c_{ij^*+1} &\le r_{ij^*} - \alpha c_{ij^*}.
\end{align*}
Because of $c_{ij^*-1} \le c_{ij^*} \le c_{ij^*+1}$ (the treatment effect is larger if the level of the treatment is higher), the first equation is equivalent to $\alpha \le \ell_{ij^*-1}$ and the second equation means $\alpha \ge \ell_{ij^*}$.

For Property~2, based on Assumption~1 (the Law of Diminishing Marginal Utility)~\cite{polleit2011can}, if $\ell_{ij^*} \le \alpha \le \ell_{ij^*-1}$ holds, the following equations also holds.
\begin{align*}
&\forall j \le j^* - 1, \ \ \alpha \le \ell_{ij^*-1} \le \ell_{ij}, \\
&\forall j \ge j^*, \ \ \alpha \ge \ell_{ij^*} \ge \ell_{ij}. 
\end{align*}
Because of $c_{ij^*-1} \le c_{ij^*} \le c_{ij^*+1}$, the first equation is equivalent to
$$\forall j \le j^* - 1, \ \ r_{ij} - \alpha c_{ij} \le r_{ij^*} - \alpha c_{ij^*}.$$
The second equation means
$$\forall j \ge j^*, \ \ r_{ij} - \alpha c_{ij} \le r_{ij^*} - \alpha c_{ij^*}.$$
Therefore, we have that $j^* = \arg \max_j r_{ij} - \alpha c_{ij}$ holds.
\end{proof}

\section{Appendix D. Proof of Theorem~\ref{theorem:MTBAP-loss}}

\begin{theorem} %\label{theorem:MTBAP-loss}
When the loss function~(7) converges, $s_{ij}$ can be used to rank $\ell_{ij}$ and $q_{ij}$ can be used to obtain the unbiased estimation of $\ell_{ij}$. 
\end{theorem}

\begin{proof}
The loss function~(7) can be rewritten as the following equation. 
\begin{align*}
L(s) &= - \left [ \sum_{j<L} \sum_{i|t_i = j + 1} \frac{1}{N_{j+1}} (q_{ij} y^r_i - q^2_{ij} y^c_i) - \right. \\ 
&\left. \ \ \ \ \ \sum_{j<L} \sum_{i|t_i = j} \frac{1}{N_j} (q_{ij} y^r_i - q^2_{ij} y^c_i) \right ] \\
&= - \sum_{j<L} \left [ \sum_{i|t_i = j + 1} \frac{1}{N_{j+1}} q_{ij} y^r_i -  \sum_{i|t_i = j} \frac{1}{N_j} q_{ij} y^r_i \right] + \\
&\ \ \ \ \ \sum_{j<L} \left [ \sum_{i|t_i = j + 1} \frac{1}{N_{j+1}} q^2_{ij} y^c_i -  \sum_{i|t_i = j} \frac{1}{N_j} q^2_{ij} y^c_i \right]
\end{align*}

Following the potential outcome framework, let $Y^r(j)$ and $Y^c(j)$ be the potential outcome in the revenue and the cost respectively when the individual receives treatment $j$. According to the definition, we have $r_{ij} = E[Y^r(j)|X = x_i]$ and $c_{ij} = E[Y^c(j)|X = x_i]$. Denote the conditional average marginal treatment effect in the revenue and the cost by $\tau^r_j(x_i), \tau^c_j(x_i)$ respectively, which satisfy
$$\tau^*_j(x_i) = E[Y^*(j+1) - Y^*(j)|X = x_i], * \in \{r,c\}.$$
Therefore, we can rewrite $\ell_{ij}$ as  $\ell_{ij} = \tau^r_j(x_i) / \tau^c_j(x_i)$. Denote $q_{ij}$ by $\kappa_j(x_i)$ in the loss function. Based on the above, the following equations hold.
\begin{align*}
\sum_{i|t_i = j} \frac{1}{N_j} q_{ij} y^r_i =& E[ \kappa_j(X_i) Y^r_i | T_i = j]\\
=& E[ \kappa_j(X_i) Y^r_i(j) | T_i = j] \ \ \ \ (\text{SUTVA})\\
=& E[\kappa_j(X_i) Y^r_i(j)] \ \ \ \ (Y^r(j)\perp T).
\end{align*}
Similar to the proof in Theorem~\ref{theorem:DUM}, we have
\begin{align*}
& \sum_{i|t_i = j + 1} \frac{1}{N_{j+1}} q_{ij} y^r_i -  \sum_{i|t_i = j} \frac{1}{N_j} q_{ij} y^r_i \\
=& E[\kappa_j(X_i) Y^r_i(j+1)] - E[\kappa_j(X_i) Y^r_i(j)] \\
=& E[\kappa_j(X_i) (Y^r_i(j+1) - Y^r_i(j))] \\
=& E_X[E[\kappa_j(X) (Y^r(j+1) - Y^r(j)) | X]] \\
=& E_X[E[ \kappa_j(X)  \tau^r_j(X)| X]] \\
=& E[\kappa_j(X_i) \tau^r_j(X_i)] \\ 
=& \frac{1}{N} \sum_i q_{ij} \tau^r_j(x_i).
\end{align*}
Therefore, the loss function can be further rewritten as
\begin{align*}
L(s) &= - \frac{1}{N} \sum_{j < L} \sum_i q_{ij}  \tau_j^r(x_i)+  \frac{1}{N} \sum_{j < L}\sum_i q^2_{ij} \tau_j^c(x_i) \\
&= - \frac{1}{N} \sum_{j<L} \sum_i [q_{ij}  \tau_j^r(x_i) - q^2_{ij} \tau_j^c(x_i)]
\end{align*}
When the loss converges, we have
$$\frac{\partial L(s)}{\partial s_{ij}} = [\tau_j^r(x_i) - 2q_{ij}\tau_j^c(x_i)] q_{ij} (1 - q_{ij}) = 0.$$
Because of $q_{ij} \in (0,1)$, the following equation holds.
$$q_{ij} = \frac{\tau^r_j(x_i)} { 2 \tau^c_j(x_i)} = \frac{1}{2}\ell_{ij}.$$
Hence, we finish the proof.
\end{proof}

\section{Appendix E. The Convergence of Direct Rank}

Direct Rank is proposed to rank ROI of individuals in the work~\cite{Du2019Improve}. Following the notations in this paper, the loss function in Direct Rank can be rewritten as Eq.~\eqref{eq:dr}. In the loss function~\eqref{eq:dr}, $\mathbb{I}_{*}$ is an indicator function and $\hbar(x_i)$ can be any machine learning model. Theorem~\ref{theorem:dr-loss} shows that this loss function~\eqref{eq:dr} is not correct, which cannot achieve the objective of ranking ROI.

\begin{align}\label{eq:dr}
\min \ \ \ \ \ \ &\ \ \ L(s) = \frac{\bar{\tau}^c}{\bar{\tau}^r} \\
s.t.\ \ \ \bar{\tau}^r &= \sum_i y^r_i p_i (\mathbb{I}_{t_i = 1} - \mathbb{I}_{t_i = 0}) \nonumber\\
\ \ \ \bar{\tau}^c &= \sum_i y^c_i p_i (\mathbb{I}_{t_i = 1} - \mathbb{I}_{t_i = 0}) \nonumber\\
\ \ \ p_i &= \frac{e^{q_i}}{\sum_j \mathbb{I}_{t_j = t_i} e^{q_j}}, \ \ \forall i \nonumber \\
\ \ \ q_i &= tanh(s_i), \ \ \forall i \nonumber \\
\ \ \ s_i &= \hbar(x_i) \in \mathbb{R}, \ \ \forall i \nonumber
\end{align}

\begin{theorem}\label{theorem:dr-loss}
The loss function~\eqref{eq:dr} in Direct Rank cannot converge to a stable extreme point.
\end{theorem}

\begin{proof}
First of all, $\bar{\tau}^r$ in Eq~\eqref{eq:dr} can be rewritten as 
\begin{align*}
\bar{\tau}^r &= \sum_{i|t_i=1} y^r_i p_i - \sum_{i|t_i=0} y^r_i p_i \\
&= \sum_{i|t_i=1}y^r_i \frac{e^{q_i}}{\sum_{j|t_j=1} e^{q_j}} - \sum_{i|t_i=0}y^r_i \frac{e^{q_i}}{\sum_{j|t_j=0} e^{q_j}} 
\end{align*}
Denote $q_i$ by $\kappa(x_i)$, and we have
\begin{align*}
&\sum_{i|t_i=1}y^r_i \frac{e^{q_i}}{\sum_{j|t_j=1} e^{q_j}} \\
=& \sum_{i|t_i=1}y^r_i \frac{e^{\kappa(x_i)}}{\sum_{j|t_j=1} e^{\kappa(x_j)}} \\
=& \frac{1}{N_1}\sum_{i|t_i=1}y^r_i \frac{e^{\kappa(x_i)}}{E[e^{\kappa(X_i)} | T_i = 1]} \\
=& \frac{E[Y^r_i e^{\kappa(X_i)} | T_i = 1]}{E[e^{\kappa(X_i)} | T_i = 1]} \\
=& \frac{E[Y^r_i(1) e^{\kappa(X_i)} | T_i = 1]}{E[e^{\kappa(X_i)} | T_i = 1]} \ \ \ (\text{SUTVA}) \\
=& \frac{E[Y^r_i(1) e^{\kappa(X_i)}]}{E[e^{\kappa(X_i)}]} \ \ \ ((Y^r(0), Y^r(1), X) \perp T) \\
\end{align*}
Similar to the proof in Theorem~\ref{theorem:DUM}, $\bar{\tau}^r$ can be further rewritten as 
\begin{align*}
\bar{\tau}^r &= \frac{E[Y^r_i(1) e^{\kappa(X_i)}]}{E[e^{\kappa(X_i)}]} - \frac{E[Y^r_i(0) e^{\kappa(X_i)}]}{E[e^{\kappa(X_i)}]} \\
&= \frac{E[(Y^r_i(1)-Y^r_i(0)) e^{\kappa(X_i)}]}{E[e^{\kappa(X_i)}]} \\
&= \frac{E[\tau^r(X_i) e^{\kappa(X_i)}]}{E[e^{\kappa(X_i)}]} \\
&= \frac{\sum_i \tau^r(x_i) e^{q_i}}{\sum_i e^{q_i}} \\
\end{align*}
Therefore, we have 
$$\bar{\tau}^c = \frac{\sum_i \tau^c(x_i) e^{q_i}}{\sum_i e^{q_i}} $$ and
$$L(s) = \frac{\sum_i \tau^c(x_i) e^{q_i}}{\sum_i \tau^r(x_i) e^{q_i}}.$$
Suppose that the loss function~\eqref{eq:dr} can converge to a stable extreme point. Hence, we can further get
\begin{align*}
\forall j, \ \ \frac{\partial L(s)}{\partial s_j} =& \frac{e^{q_j} (1 - q_j^2)}{(\sum_i \tau^r(x_i) e^{q_i})^2} \left[\tau^c(x_j) \sum_i \tau^r(x_i) e^{q_i} - \right.\\
&\left. \tau^r(x_j)\sum_i \tau^c(x_i) e^{q_i} \right] = 0.
\end{align*}
It is equivalent to 
$$\forall j, \frac{\tau^r(x_j)}{\tau^c(x_j)} = \frac{\sum_i \tau^r(x_i) e^{q_i}}{\sum_i \tau^c(x_i) e^{q_i}}.$$
The above equation means that all the individuals have the same value of ROI. It does not hold obviously. Hence, we finish the proof.
\end{proof}

\section{Appendix F. Supplementary Experimental Results}

\begin{table}[htbp]%调节图片位置，h：浮动；t：顶部；b:底部；p：当前位置
    \centering
    \caption{Summary of Evaluation Results}
    \label{tab:evaluation}  
\begin{tabular}{*{6}{c}}
  \toprule
  \multirow{2}*{Model} & \multicolumn{2}{c}{AUUC} & \\ \cmidrule(lr){2-3}
  & CRITEO-UPLIFT v2 & Marketing data  \\
  \midrule
  S-Learner &  0.8440 $\pm$ 0.0054 &  0.7826 $\pm$ 0.0077 \\
  X-Learner &  0.8304 $\pm$ 0.0377  & 0.7948 $\pm$ 0.0102 \\
  Causal Forest & 0.8485 $\pm$ 0.0025 & 0.7468 $\pm$  0.0059 \\
  DUM & 0.8749 $\pm$ 0.0155  & 0.8077 $\pm$ 0.0003 \\
  \bottomrule
 \toprule
  \multirow{2}*{Model} & \multicolumn{2}{c}{AUCC} & \\ \cmidrule(lr){2-3}
  & CRITEO-UPLIFT v2 & Marketing data  \\
  \midrule
  TPM-SL & 0.7561 $\pm$ 0.0113 & 0.7153 $\pm$  0.0142\\
  Direct Rank & 0.7562 $\pm$ 0.0131 & 0.7414 $\pm$ 0.0101\\
  DRP & 0.7739 $\pm$ 0.0002 & 0.7666 $\pm$  0.0003\\
  \bottomrule
  \toprule
  \multirow{2}*{Model} & \multicolumn{1}{c}{MT-AUCC} \\ \cmidrule(lr){2-2}
  & Marketing data &\\
  \midrule
  TPM-SL & 0.7250 $\pm$ 0.0143  \\
  TPM-CF & 0.7395 $\pm$ 0.0159   \\
  DPM & 0.7860 $\pm$ 0.0007  \\
  \bottomrule
\end{tabular}
\end{table}

Table~\ref{tab:evaluation} presents the detailed experimental results of different models. First of all, our models perform best in these evaluation metrics. 
In addition, all the models proposed in this paper are very stable and have very low variance compared with other existing works. This is because our models can make a direct prediction for the final objective, and always converge to a stable extreme point.

\end{document}